%% file: main.tex
\documentclass[sigconf,nonacm]{acmart}

\AtBeginDocument{%
  \providecommand\BibTeX{{%
    \normalfont B\kern-0.5em{\scshape i\kern-0.25em b}\kern-0.8em\TeX}}}

\usepackage[utf8]{inputenc} 
\usepackage[T1]{fontenc}    
\usepackage{xr-hyper}
\usepackage{url}            
\usepackage{booktabs}       
\usepackage{amsfonts}       
\usepackage{nicefrac}       
\usepackage{microtype}      

\usepackage{enumitem}

\usepackage{graphicx}
\usepackage{subfigure}
\usepackage{booktabs} 
\usepackage{amsmath,amsthm,amsfonts}
\usepackage{enumitem}
\usepackage{color}
\usepackage{overpic}
\usepackage{balance}

\usepackage{multirow}
\usepackage{varwidth} 

\usepackage[section]{algorithm}
\usepackage{algorithmicx,algpseudocode}

\usepackage{wrapfig}

\newcommand\CONDITION[2]%
  {\begin{tabular}[t]{@{}l@{}l@{}}
     #1&#2
   \end{tabular}%
  }
\algdef{SE}[IF]{If}{EndIf}[1]%
  {\algorithmicif\ \CONDITION{#1}{\ \algorithmicthen}}%
  {\algorithmicend\ \algorithmicif}%
\algloopdefx{NFor}[1]{\textbf{for} #1 \textbf{do}}
\algloopdefx{NIf}[1]{\algorithmicif\ \CONDITION{#1}{\ \algorithmicthen}}

\numberwithin{equation}{section}

\theoremstyle{plain}
\newtheorem{theorem}{Theorem}
\newtheorem{lemma}[theorem]{Lemma}

\theoremstyle{definition}
\newtheorem{definition}{Definition}
\theoremstyle{remark}
\newtheorem{remark}{Remark}
\newtheorem{heuristic}{Heuristic}

\newcommand{\improved}{{\sc Batched}}
\newcommand{\baseline}{{\sc Basic}}
\newcommand{\uniform}{{\sc Uniform}}
\newcommand{\shuttle}{{\sc Shuttle}}
\newcommand{\kdd}{{\sc Kdd}}
\newcommand{\eps}{\epsilon}

\DeclareMathOperator{\poly}{poly}

\DeclareMathOperator{\E}{\mathbb{E}}
\DeclareMathOperator{\R}{\mathbb{R}}
\DeclareMathOperator*{\argmax}{arg\, max}
\DeclareMathOperator*{\argmin}{arg\, min}

\newcommand{\abs}[1]{\left| #1 \right|}

\numberwithin{equation}{section}

\usepackage{tikz}
\usetikzlibrary{calc,decorations.pathreplacing}

\newcommand{\tikzmark}[1]{\tikz[overlay,remember picture] \node (#1) {};}

\newcommand*{\AddNoteTight}[4]{%
    \begin{tikzpicture}[overlay, remember picture]
        \draw [decoration={brace,amplitude=0.5em},decorate,thick]
            ($(#3)!(#1.north)!($(#3)-(0,1)$)+(0.0,0)$) --  
            ($(#3)!(#2.south)!($(#3)-(0,1)$)+(0.0,0)$)
                node [align=center, text width=1cm, pos=0.5, anchor=west] {#4};
    \end{tikzpicture}
}

\allowdisplaybreaks
\def\@textbottom{\vskip \z@ \@plus 2sp}

\begin{document}
\fancyhead{}
\title{Learning to Cluster via Same-Cluster Queries}

\author{Yi Li}
\affiliation{%
  \institution{Nanyang Technological University}
	\streetaddress{21 Nanyang Link}
	\city{}
	\country{Singapore}
	\postcode{637371}
}
\email{yili@ntu.edu.sg}
\author{Yan Song}
\affiliation{%
  \institution{Indiana University Bloomington}
  \streetaddress{700 North Woodlawn Avenue}
	\city{Bloomington}
	\state{IN}
	\country{USA}
	\postcode{47408}
}
\email{songyan@iu.edu}
\author{Qin Zhang}
\affiliation{%
  \institution{Indiana University Bloomington}
	\streetaddress{700 North Woodlawn Avenue}
	\city{Bloomington}
	\state{IN}
	\country{USA}
	\postcode{47408}
}
\email{qzhangcs@indiana.edu}


\begin{abstract}
  We study the problem of learning to cluster data points using an oracle which can answer same-cluster queries. Different from previous approaches, we do not assume that the total number of clusters is known at the beginning and do not require that the  true clusters are consistent with a predefined objective function such as the $K$-means. These relaxations are critical from the practical perspective and, meanwhile, make the problem more challenging. We propose two algorithms with provable theoretical guarantees and verify their effectiveness via an extensive set of experiments on both synthetic and real-world data.
\end{abstract}

\ccsdesc[500]{Theory of computation~Semi-supervised learning}
\ccsdesc[500]{Theory of computation~Facility location and clustering}

\keywords{clustering; weak supervision; same-cluster oracle}

\maketitle

\input{intro}

\input{prelim}
  
\input{baseline}

\input{improved}

\input{exp}

\section*{Acknowledgments}
The authors would also like to thank the anonymous referees for their valuable comments and helpful suggestions. 
Yan Song and Qin Zhang are supported in part by NSF IIS-1633215 and CCF-1844234.



\bibliographystyle{ACM-Reference-Format}
\bibliography{paper}

\clearpage

\appendix
\input{prelim_supp}

\input{basic_analysis}

\input{omitted_analysis_proofs}

\input{noisy}
\section{Omitted Tables of Experiment Results}\label{sec:query_complexity_tables}
\input{query_complexity_tables}

\clearpage

\end{document}

%% file: intro.tex
\section{Introduction}
\label{sec:intro}

Clustering is a fundamental problem in data analytics and sees numerous applications across many areas in computer science.
However, many clustering problems are computationally hard, even for approximate solutions.  To alleviate the computational burden, Ashtiani et al.~\cite{AKB16} introduced {\em weak supervision} into the clustering process. More precisely, they allow the algorithm to query an oracle which can answer {\em same-cluster} queries, that is, whether two elements (points in the Euclidean space) belong to the same cluster.  The oracle can, for example, be a domain expert in the setting of crowdsourcing.  It was shown in \cite{AKB16} that the $K$-means clustering can be solved computationally efficiently with the help of a small number of same-cluster queries using the oracle.


The initial work by Ashtiani et al.~\cite{AKB16} relies on a strong ``data niceness property'' called the {\em $\gamma$-margin}, which requires that for each cluster $C$ with center $\mu$, the distance between any point $p \not\in C$ and $\mu$ needs to be larger than that between any point $p \in C$ and $\mu$ by at least a constant multiplicative factor $\gamma > 1$ sufficiently large.  This assumption was later removed by Ailon et al.~\cite{ABJK18}, who gave an algorithm which, given a parameter $K$, outputs a set $C$ of $K$ centers attaining a $(1+\eps)$-approximation of the $K$-means objective function.  There has been follow-up work by Chien et al.~\cite{CPM18} in the same setting, replacing the $\gamma$-margin assumption with a new assumption on the size of the clusters.  However, there are still two critical issues in this line of approach:
\begin{enumerate}[leftmargin=20pt]
\item In \cite{ABJK18,CPM18}, it is assumed that the total number of clusters $K$ in the {\em true clustering} is known to the algorithm, which is unrealistic in many cases.

\item In \cite{ABJK18,CPM18}, it is assumed that the optimal solution to the $K$-means clustering is {\em consistent} with the true clustering (i.e., when we have ground truth labels for all points). This assumption is highly problematic, since the ground-truth clustering can be arbitrary and very different from an optimal solution with respect to a fixed objective function such as the $K$-means.
\end{enumerate}

In this paper, we aim to remove both assumptions. We shall design algorithms that find the approximate centers of the true clusters using same-cluster queries, in the setting that we do {\em not} know in advance the number of clusters in the true clustering and the true clustering has {\em no} relationship with the optimal solution of a certain objective function.  We obtain our result at a small cost: our sample-based algorithm may not be able to find the centers of the small clusters which are very close to some big identified clusters; we shall elaborate on this shortly. 
\vspace{2mm}

\noindent{\bf Problem Setup.} \  We denote the clusters by $X_1,X_2,\dots$, which are point sets in the Euclidean space. For each $X_i$, let $\mu_i = \frac{1}{|X_i|}\sum_{x\in X_i} x$ be its centroid. For simplicity we refer to $X_i$ as ``cluster $i$'' or ``$i$-th cluster''. The number of clusters is unknown at the beginning. Our goal is to find all ``big'' clusters and their approximate centroids $\tilde \mu_i$. 

To specify what we mean by finding all big clusters, we need to introduce a concept called {\em reducibility}.  Several previous studies~\cite{KSS10,JKS14,ABJK18} on the $K$-means/median clustering problem assumed that after finding $K$ main clusters, the residual ones satisfy some reducibility condition, which states that using the $K$ discovered cluster centers to cover the residual ones would only increase the total cost by a small $(1+\eps)$ factor. Similarly, we consider the following reducibility condition, based on the usual cost function.

\begin{definition}[cost function] Suppose that $X$ and $C$ are the set of data points and centers, respectively. The cost of covering $X$ using $C$ is $\Phi(X,C) = \sum_{x\in X} \min_{c\in C} \|x-c\|^2$, where $\|\cdot\|$ denotes the Euclidean norm.
When $C=\{c\}$ is a singleton, we also write $\Phi(X,\{c\})$ as $\Phi(X,c)$. When $C=\emptyset$, we define $\Phi(X,C) = |X|\cdot \sup_{x,y\in X}\|x-y\|^2$.
\end{definition}

\begin{definition}[$\eps$-reducibility] \label{def:reducibility}
Let $X_1, X_2, \ldots$ be true clusters in $X$.  Let $I$ be a subset of indices of clusters in $X$. We say the clusters in $X$ are $\eps$-reducible w.r.t.\ $I$ if it holds for each $\ell\not\in I$ that
\[
 \Phi\left(X_\ell, \bigcup_{i\in I} \{\mu_{i}\}\right) \leq \eps \sum_{i\in I} \Phi(X_i, \mu_i).
\]
\end{definition}
Intuitively, this reducibility condition states that the clusters outside $I$ will be {\em covered} by the centers of the clusters in $I$ with only a small increase in the total cost.  

Our goal is to find a subset of indices $I$ such that the set of true clusters in $X$ is $\eps$-reducible w.r.t.\ $I$.   

Note that in this formulation, we do {\bf not} attempt to optimize any objective function; instead, we just try to recover the approximate centroids of a set of clusters to which all other clusters are reducible.  Inevitably we have to adopt a distance function in our definition of irreducibility, and we choose the widely used $D^2$ (Euclidean-squared) distance function so that we can still use the $D^2$-sampling method developed and used in the earlier works \cite{AV07,JKS14,ABJK18}. The $D^2$-sampling will be introduced in Definition~\ref{def:D^2 sampling}.

\begin{table}
\caption{Notations}\label{tab:notations}
\begin{tabular}{|l|l|}
\hline
$X$ & set of input points \\
\hline
$r$ & index of the round \\
\hline
$I$ & set of indices of recovered clusters so far\\
\hline
$k$ & number of recovered clusters so far \\
\hline
$Q$ & set of indices of newly discovered clusters\\ 
& in the current round \\
\hline
$q$ & size of $Q$; number of newly discovered clusters\\ 
& in the current round \\
\hline
$S$ & multiset of samples. Each sample has the form of \\
    & $(x,i)$, where $x$ is the point and $i$ the cluster index.\\
\hline
$x_j^\ast$ & reference point in cluster $j$ for rejection sampling \\
\hline
$S_j$ & multiset of uniform samples in cluster $j$ returned \\
& by rejection sampling. Note that $S\neq \bigcup_j S_j$. \\
\hline
$W$ & set of indices of clusters to recover\\
\hline
$K$ & total number of recovered clusters at the end \\
\hline
$L$ & total number of discovered clusters at the end \\
\hline
\end{tabular}
\end{table}

To facilitate discussion, we list in Table~\ref{tab:notations} the commonly used notations in our algorithms and analyses. A cluster $i$ is said to be ``discovered'' if any point in $X_i$ has been sampled and ``recovered'' when the approximate centroid $\tilde{\mu}_i$ is computed. As discussed in the preceding paragraph, it is possible that the total number of recovered clusters, $K$, is less than the total number of discovered clusters, $L$, and that $L$ is less than the true number of clusters.

\vspace{2mm}

\noindent{\bf Our Contributions.} \
We provide two clustering algorithms with theoretically proven bounds on the total number of oracle queries in the circumstance of no prior knowledge on the number of clusters and no predefined objective function. To the best of our knowledge, these are the first algorithms for the clustering problem of its kind. Both our algorithms output $(1+\eps)$-approximate centers for all recovered clusters.
 Our first algorithm makes $\tilde{O}(\eps^{-4}K^2L^2)$ queries (Section~\ref{sec:basic}) and the second algorithm makes $\tilde{O}(\eps^{-4}KL^2)$ queries (Section~\ref{sec:improve}).\footnote{In $\tilde{O}(\cdot), \tilde{\Omega}(\cdot), \tilde{\Theta}(\cdot)$ we use `$\tilde{~}$' to hide logarithmic factors.  The exact query complexities can be found in the Theorems~\ref{thm:basic} and \ref{thm:main}.}

We also conduct an extensive set of experiments demonstrating the effectiveness of our algorithms on both synthetic and real-world data; see Section~\ref{sec:exp}. 

We further extend our algorithms to the case of a noisy same-cluster oracle which errs with a constant probability $p<1/2$. This extension has been deferred to Appendix~\ref{sec:noisy oracle} of the full version of this paper.

We remark that since our algorithms target {\em sublinear} (i.e., $o(\abs{X})$) number of oracle queries, in the general case where the shape of clusters can be arbitrary, it is impossible to classify correctly all points in the datasets. But the approximate centers outputted by the algorithms can be used to {\em efficiently and correctly} classify any newly inserted points (i.e., database queries) as well as existing database points (when needed), using a natural heuristic (Heuristic~\ref{heu:classify}) that we shall introduce in Section~\ref{sec:exp}. Our experiments show that most points can be classified using only {\em one} additional oracle query.

\vspace{2mm}

\noindent{\bf Related Work.} \
As mentioned, the semi-supervised active clustering framework was first introduced in \cite{AKB16}, where the authors considered the $K$-means clustering under the $\gamma$-margin assumption.  Ailon et al.~\cite{ABJ18,ABJK18} proposed approximation algorithms for the $K$-means and correlation clustering that compute a $(1+\eps)$-approximation of the optimal solution with the help of same-cluster queries.  Chien et al.~\cite{CPM18} studied the $K$-means clustering under the same setting as that in \cite{ABJK18}, but used uniform sampling instead of $D^2$-sampling and worked under the assumption that no cluster (in the true clustering) has a small size.  Saha and Subramanian~\cite{SS19} gave algorithms for correlation clustering with the same-cluster query complexity bounded by the optimal cost of the clustering. Gamlath et al.~\cite{GHS18} extended the $K$-means problem from Euclidean spaces to general finite metric spaces with a strengthened guarantee that recovered clusters largely overlap with respect to the true clusters.  All these works, however, assume that the ground truth clustering (known by the oracle) is consistent with the target objective function, which is unrealistic in most real world applications. 

Bressan et al.~\cite{BCLP20} considered the case where the clusters are separated by ellipsoids, in contrast to the balls as suggested by the usual $K$-means clustering objective. However, their algorithm still requires the knowledge of $K$ and has a query complexity that depends on $n$ (although logarithmically), which we avoid in this work.

Mazumdar and Saha studied clustering $n$ points into $K$ clusters with a noisy oracle \cite{MS17a} or side information \cite{MS17b}.  The noisy oracle gives incorrect answers to same-cluster queries with probability $p < 1/2$ (so that majority voting still works).  The side information is the similarity score between each pair of data points, generated from a  distribution $f_+$ if the pair belongs to the same cluster and from another distribution $f_-$ otherwise. Algorithms proposed in these papers guarantee to recover the true clusters of size at least $\Omega(\log n)$, however, with query complexities at least $\Omega(K n)$, much larger than what we are interested to achieve in this paper.

Huleihel et al.~\cite{HMMP19} studied the overlapping clustering with the aid of a same-cluster oracle.  Suppose that $A$ is an $n \times K$ clustering matrix whose $i$-th row is the indicator vector of the cluster membership of the $i$-th element.  The task is to recover $A$ from the similarity matrix $A A^T$ using a small number of oracle queries. 

Finally, we note that same-cluster queries have been used extensively for {\em entity resolution} (or, {\em de-duplication})~\cite{WKF12, WLK13, VBD14, VG15, FSS16}.

%% file: prelim.tex
\section{Preliminaries}

In this paper we consider point sets in the canonical Euclidean space $(\R^d,\|\cdot\|)$. The geometric centroid, or simply centroid, of a finite point set $X$ is defined as $\mu(X) = \frac{1}{|X|} \sum_{x\in X} x$. It is known that $\mu(X)$ is the minimizer of the $1$-center problem $\min_c \sum_{x\in X} \|x-c\|^2$.  The next lemma provides a guarantee on approximating the centroid of a cluster using uniform samples.
\begin{lemma}[{\cite{IKI94}}]\label{lem:centroid}
Let $S$ be a set of points obtained by independently sampling M points uniformly at random with replacement from a point set $X\in \R^d$. Then for any $\delta>0$, it holds that 
\[
\Pr\left\{ \Phi(S, \mu(S)) \leq \left(1+1/(\delta M)\right)\Phi(X, \mu(X)) \right\}\geq 1-\delta.
\]
\end{lemma}

We define the $D^2$-sampling of a point set $X$ with respect to a point set $C$ as follows.
\begin{definition}[$D^2$-sampling]\label{def:D^2 sampling}
The $D^2$-sampling of a point set $X$ with respect to a point set $C$ returns a random point $p\in X$ subject to the distribution defined by $\Pr\{p=x\} = \Phi(\{x\},C)/\Phi(X,C)$ for all $x\in X$.
\end{definition}

In this paper we use a same-cluster oracle, that is, given two data points $x,y\in X$, $\textsc{Oracle}(x,y)$ returns true if $x$ and $y$ belong to the same cluster and false otherwise. For simplicity of the algorithm description, we shall instead invoke the function $\textsc{Classify}(x)$ to obtain the cluster index of $x$, which can be easily implemented using the same-cluster oracle, as shown in  Algorithm~\ref{alg:classify}. This implies that the number of oracle queries is at most $L$ times the number of samples.

\begin{algorithm}[t]
\caption{$\textsc{Classify}(x)$. The overall algorithm which maintains the number $L$ of discovered clusters and a representative point $z_i$ for each $i=1,\dots,L$.}\label{alg:classify}
\begin{algorithmic}
	\NFor{$i=1$ \textbf{to} $L$}
		\NIf{$\Call{Oracle}{x,z_i}$}
			\State \Return $i$
	\State $L\gets L+1$
	\State $z_L \gets x$
	\State \Return $L$
\end{algorithmic}
\end{algorithm}

%% file: baseline.tex

\section{Basic Algorithm}
\label{sec:basic}

\paragraph{Algorithm.} Despite the fact that the algorithm in~\cite{ABJK18} cannot be used directly because $K$ is unknown to us, we briefly review the algorithm below. The algorithm gradually ``discovers'' and then ``recovers'' more and more clusters by taking $D^2$-samples. It recovers $K$ clusters in $K$ rounds, one in each round. To recover one cluster, it takes sufficient $D^2$-samples w.r.t.\ the approximate centers $\{\tilde \mu_i\}_{i\in I}$, where $I$ is the set of the recovered clusters, and finds the largest unrecovered cluster $j$. For this cluster $j$, it obtains sufficient uniform samples via rejection sampling and  invokes Lemma~\ref{lem:centroid} to 
compute an approximate centroid $\tilde\mu_j$. Then it includes $j$ in the set $I$ of recovered clusters and proceeds to the next round. It is shown that $\poly(K/\eps)$ samples each round can achieve the failure probability of $O(1/K)$ and taking a union bound over all $K$ rounds yields a constant overall failure probability.

There are two major difficulties of adapting this algorithm to our setting where $K$ is unknown.
\begin{enumerate}[leftmargin=20pt]
\item The number of rounds is unknown. It is not clear when our algorithm should terminate, i.e., when we are confident that there are no more irreducible clusters. The failure probability of each round also needs to be redesigned and cannot be $O(1/K)$.

\item Since $K$ is unknown, we cannot predetermine the number of samples to use and have to maintain dynamically various stopping criteria. For instance, it is subtle to determine which one the largest cluster is. If we stop too early, we may identify a cluster that is actually small and will need a large number of samples in order to obtain enough uniform samples from this cluster when doing rejection sampling; if we stop too late, we can make a more accurate decision but may have already taken too many samples.
\end{enumerate}

To address the first issue, we observe that $\Omega(\eps^{-1}\log K)$ samples is sufficient to test whether there are $\eps$-irreducible clusters left with constant probability. We also assign the failure probability of the $r$-th round to be $a_r$ such that $\sum_{r=1}^\infty a_r$ is a small constant.

To address the second issue, observe that the $\eps$-irreducibility condition implies that all unrecovered clusters together have a $\Phi$-value of $\Omega(\eps)$ and thus the largest unrecovered cluster has a $\Phi$-value of $\Omega(\eps/q)$, where $q$ is the number of newly discovered clusters. We can show that $\tilde{O}(q/\eps)$ samples suffices to ensure the identification of a cluster with $\Phi$-value at least $\Omega(\eps/q)$ and thus we can, since the value of $q$ increases as the number of samples grows, keep sampling until $\tilde{O}(q/\eps)$ samples are obtained. Note that this is a dynamic criterion in contrast to the predetermined one in~\cite{ABJK18}. We then choose the largest cluster $j$ and carry out the rejection sampling and estimate the approximate centroid $\tilde\mu_j$ with a careful control over the number of samples.

\begin{algorithm}[t]
\caption{The Basic Algorithm. 
}\label{alg:basic}
\begin{algorithmic}[1]
\State $I\gets \emptyset$ 
\State $k\gets 0$, $r\gets 0$
\Repeat
	\State $r\gets r + 1$
	\State $Q\gets \emptyset$, $S\gets \emptyset$
	\State $T_1\gets 8\eps^{-1}\ln(10(k+1))$\label{alg:basic ph1a} \label{alg:basic T_1} \tikzmark{ph1a}
	\While{$Q=\emptyset$ and $|S|\leq T_1$}
		\State $(x,Q,S)\gets \textsc{Sample}(I,\{\tilde\mu_i\}_{i\in I},Q,S)$ \tikzmark{right1}
	\EndWhile   \label{alg:basic ph1b} \tikzmark{ph1b}
	\If {$Q \neq \emptyset$}
		\While{$|S|\leq 96|Q|\ln(10(k+|Q|))/\eps$} \label{alg:basic ph2a} \tikzmark{ph2a} \tikzmark{right2}
			\State $(x,Q,S)\gets \textsc{Sample}(I,\{\tilde\mu_i\}_{i\in I},Q,S)$
		\EndWhile
		\State $j \gets \argmax_i |\{u\!\in\! S: u \!=\! (x,i)\}|$ \label{alg:basic ph2b} \tikzmark{ph2b}
		\State $q\gets |Q|$ \label{alg:basic ph3a} \tikzmark{ph3a}
		\State $T_2 \gets 2^{12}q\ln(10(k+q))/\eps^2$ \label{alg:basic T_2}
		\While{$|S|\leq T_2$} 
			\State $(x,Q,S)\gets \textsc{Sample}(I,\{\tilde\mu_i\}_{i\in I},Q,S)$ \tikzmark{right3}
		\EndWhile \label{alg:basic ph3b} \tikzmark{ph3b}
		\State $x_j^\ast \gets \argmin_{x:(x,j)\in S} \Phi(\{x\}, \{\tilde\mu_i\}_{i\in I})$ 
		\State $S_j\gets\emptyset$ 
		\State $T_3 \gets 20\eps^{-1}r\ln^2(10r)$   \label{alg:basic T_3} \label{alg:basic ph4a} \tikzmark{ph4a}
		\State $S_j \gets \textsc{RejSamp}(I,\{\tilde\mu_i\}_{i\in I},T_3,\{j\},\{x_j^\ast\})$ \tikzmark{right4}
		\State $\tilde\mu_j \gets (1/|S_j|)\sum_{x\in S_j} x$ \label{alg:basic ph4b} \tikzmark{ph4b}
		\State $k\gets k+1$
		\State $I\gets I\cup \{j\}$
	\EndIf
\Until{$Q=\emptyset$} \Comment{no new cluster is discovered}
\AddNoteTight{ph1a}{ph1b}{right1}{Phase\\ 1}
\AddNoteTight{ph2a}{ph2b}{right2}{Phase\\ 2}
\AddNoteTight{ph3a}{ph3b}{right3}{Phase\\ 3}
\AddNoteTight{ph4a}{ph4b}{right4}{Phase\\ 4}
\end{algorithmic}
\end{algorithm}

\begin{algorithm}[t]
\caption{$\textsc{Sample}(I,\{\tilde\mu_i\}_{i\in I},Q,S)$: Returning a single $D^2$-sample w.r.t.\ recovered clusters $I$}\label{alg:sample}
\begin{algorithmic}[1]
		\State $x\gets $ a point returned by $D^2$-sampling (w.r.t.\ $\{\tilde\mu_i\}_{i\in I}$)
		\State $j\gets \textsc{Classify}(x)$
		\NIf{$j\not\in I$}
			\State $Q\gets Q\cup\{j\}$
		\State $S\gets S\cup\{(x,j)\}$
		\State \Return $(x,Q,S)$
\end{algorithmic}
\end{algorithm}
\hfill
\begin{algorithm}[t]
\caption{$\textsc{RejSamp}(I,\{\tilde\mu_i\}_{i\in I},T,W,\{x_j^\ast\}_{j\in W})$: Rejection sampling on heavy clusters}\label{alg:rejection}
\begin{algorithmic}[1]
		\Repeat
			\State $x\gets $ a point returned by $D^2$-sampling w.r.t.~$\{\tilde\mu_i\}_{i\in I}$
			\State $j\gets \textsc{Classify}(x)$
			\NIf{$j\in W$}
				\State Add $x$ to $S_j$ with probability $\frac{\eps}{128}\cdot\frac{\Phi(\{x_j^\ast\},\{\tilde\mu_i\}_{i\in I})}{\Phi(\{x\},\{\tilde\mu_i\}_{i\in I})}$ 
		\Until{$|S_j|\geq T$ for all $j\in W$}
		\State \Return $\{S_j\}_{j\in W}$
\end{algorithmic}
\end{algorithm}

We present our algorithm in Algorithm~\ref{alg:basic} without attempts to optimize the constants. In our algorithm, each round is divided into four phases. Phase 1 tests whether there are new clusters to recover, using $O(\eps^{-1}\ln k)$ samples, where $k$ is the number of the recovered clusters. If no new clusters are found, the algorithm would terminate. Phase 2 samples more points and finds the largest cluster discovered so far, which we shall aim to recover in the two remaining phases. Phase 3 samples more points and finds a reference point $x_j^\ast$ to be used in the rejection sampling in the next phase. Phase 4 executes rejection sampling using the reference point $x_j^\ast$ such that the returned samples are uniform from the cluster $j$. We need to collect $\tilde O(k/\eps)$ uniform samples before calculating the approximate centroid $\tilde\mu_j$ for cluster $j$. We then declare that cluster $j$ is recovered and start the next round.

We would like to remark that the $D^2$-sampling which our algorithm uses has a great advantage over uniform sampling~\cite{CPM18} when there are small and faraway clusters. Uniform sampling would require a large number of samples to find small clusters, regardless of their locations. On the other hand, faraway clusters are clearly irreducible and are rightfully expected to be recognized. The $D^2$-sampling captures the distance information such that those small faraway clusters may have a large $\Phi$-value and can be found with much fewer samples.

\paragraph{Analysis.}  Throughout the execution of the algorithm, we keep the loop invariant that 
\begin{equation}\label{eqn:good_approx_centroid}
\Phi(X_i,\tilde\mu_i)\leq (1+\eps)\Phi(X_i,\mu_i)
\end{equation}
for all $i\in I$, i.e., the approximate centers are good for all recovered clusters.


It is clear that the main algorithm will terminate. We prove that the algorithm is correct, i.e., the algorithm will find all clusters with $\tilde\mu_i$'s being all good approximate centroids. We first need a lemma showing that the undiscovered clusters are heavy under the assumption of $\eps$-reducibility. All proofs are postponed to Appendix~\ref{sec:basic proofs}.

\begin{lemma}\label{lem:undiscovered_is_heavy} Suppose the $\eps$-reducibility constraint w.r.t.\ $I$ does not hold. It holds for $\eps\in (0, 1]$ that 
$\sum_{i\not\in I}\Phi(X_i, \tilde C)/\sum_{i}\Phi(X_i, \tilde C) \geq \eps/4$.
\end{lemma}

Using the preceding lemma, we are ready to show that Phase 1 of Algorithm~\ref{alg:basic} discovers all heavy clusters with high probability.
\begin{lemma}\label{lem:correctness basic}
Upon the termination of Algorithm~\ref{alg:basic}, with probability at least $0.95$, the $\eps$-reducibility condition w.r.t.\ $I$ is satisfied and an approximate centroid $\tilde\mu_i$ that satisfies~\eqref{eqn:good_approx_centroid} is obtained for every $i\in I$.
\end{lemma}
%
%

Next we upper bound the number of samples. For each cluster $j\not\in I$ define its conditional sample probability
\begin{equation}\label{eqn:def p_j}
\textstyle p_j = \Phi(X_j,C) / \sum_{i\notin I} \Phi(X_i, C),
\end{equation}
where $C = \{\tilde\mu_i\}_{i\in I}$ is the set of approximate centers so far. For each $p_j$, we also define an empirical approximation  $\hat p_j = s_j/\sum_{i\notin I} s_i$, where $s_i$ denotes the number of samples seen so far which belong to cluster $i$.

The next lemma shows that the new clusters we find in Phase 2 are heavy.
\begin{lemma}\label{lem:each p_i basic}
With probability at least $1-2/(100(k+q)^2)$, it holds that $p_j\geq 1/(3q)$ for all $j\in W$.
\end{lemma}

To analyse Phases 3 and 4, we define an auxiliary point set
\begin{equation}\label{eqn:Y_j_def}
Y_j = \left\{y\in X_j: \Phi(\{y\},C)/\Phi(X_j,C)\leq 2/|X_j|\right\}.
\end{equation}
Points in $Y_j$ are good pivot points for the rejection sampling procedure in Phase 4. We first show that we can find a good pivot point $x_j^\ast$ in Phase 3.
\begin{lemma}\label{lem:Y_j basic}
Assume that the event in Lemma~\ref{lem:each p_i basic} occurs. Then $x_j^\ast\in Y_j$ with probability at least $1-1/(100(k+q)^2)$.
\end{lemma}

We then show that, with a good pivot point $x_j^\ast$, the rejection sampling procedure in Phase 4 obtains sufficiently many uniform samples from cluster $j$ so that we can calculate a good approximate center later.
\begin{lemma}\label{lem:rejection sampling basic}
Assume that the event in Lemma~\ref{lem:Y_j basic} occurs.
By sampling $s =2^{23}\eps^{-4}qr\ln^2(10r)$ points in total, with probability at least $1-\exp(-8r)$, every cluster $j\in W$ has $T_3 = 20\eps^{-1}r\ln^2(10r)$ samples returned by the rejection sampling procedure.
\end{lemma}

Combining Lemmata~\ref{lem:correctness basic}, \ref{lem:each p_i basic}, \ref{lem:Y_j basic} and \ref{lem:rejection sampling basic}, we arrive at the main conclusion of the basic algorithm.
\begin{theorem}\label{thm:basic}
With probability at least $0.9$, Algorithm~\ref{alg:basic} finds a set $I$ of clusters that satisfies Definition~\ref{def:reducibility}, and obtains for every $i \in I$ an approximate centroid $\tilde\mu_i$ that satisfies~\eqref{eqn:good_approx_centroid}, using $O(\eps^{-4}K^2 L^2\log^2 L)$ same-cluster queries in total, where $K$ is the total number of recovered clusters and $L$ is the total number of discovered clusters.
\end{theorem}
%

\begin{remark}
When the clusters are $\eps$-irreducible with respect to any subcollection of the clusters, we shall recover all $L$ clusters using $O(\eps^{-4}L^4\log^2 L)$ same-cluster queries.
\end{remark}

%% file: improved.tex
\section{Improved Algorithm}
\label{sec:improve}

In this section, we improve the basic algorithm by allowing the recovery of multiple clusters in each round. A particular case in which the basic algorithm would suffer from a prodigal waste of samples is that there are in total $K$ clusters of approximately the same size and they are $\eps$-irreducible with respect to any subset $I\subset [K]$. In such case, our basic algorithm requires a sample of size $\tilde O(K/\eps)$ from that cluster, which in turn requires a global sample of size $\tilde O(K^2/\eps)$. Summing over $K$ rounds leads to a total number $O(K^3/\eps)$ of samples. However, with $\tilde O(K^2/\eps)$ global samples, we can obtain $\tilde O(K/\eps)$ samples from every cluster and recover all clusters in the same round. This reduces a factor of $K$ in the total number of samples. The goal of this section is to improve the sample/query complexity of the basic algorithm by a factor of $K$ even in the \emph{worst} case.

To recover an indefinite number of clusters in each round, it is a natural idea to generalize the basic algorithm to identifying a set $W$ of clusters, instead of the biggest one only, from which we shall obtain uniform samples via rejection sampling and calculate the approximate centroids $\{\tilde\mu_j\}_{j\in W}$. Here we face an ever greater challenge as to how to determine $W$. The na\"ive idea of including all clusters of $\Phi$-value at least $\Omega(\eps/r)$ runs into difficulty: as we sample more points, the value of $r$ may increase and we may need to include more and more points, so when do we stop? We may expect that the value $r$ will stabilize after sufficiently many samples, but it is difficult to quantify ``stable'' and control the number of samples needed.

Our solution to overcome the above-mentioned difficulty is to split the clusters into bands. To explain this we introduce a few more notations. For each unrecovered cluster $j\notin I$, recall that we defined in~\eqref{eqn:def p_j} its conditional sample probability 
$p_j = \Phi(X_j,C)/\sum_{i\notin I} \Phi(X_i, C)$; we also define its empirical approximation to be $\hat p_j = s_j/\sum_{i\notin I} s_i$, where $s_i$ denotes the number of samples seen so far which belong to the cluster $i$. We split $\{\hat p_i\}_{i\notin I}$ into $L+1$ bands for $L = 3\log q$, where the $\ell$-th band is defined to be 
\[
 B_\ell = \{i\not\in I: 2^{-\ell} < \hat p_i\leq 2^{-\ell+1} \}
\]
for $\ell\leq L$ and the last band $B_{L+1}$ consists of all the remaining clusters $i$ (i.e., $\hat p_i\leq 1/q^3$). We say a band $B_\ell$ is \emph{heavy} if $\sum_{i\in B_\ell} \hat p_i \geq 1/(3L)$. The values $\hat p_i$ and the bands are dynamically adjusted as the number of samples grows.

Instead of focusing on individual clusters and identifying the heavy ones, we shall identify the heavy bands at an appropriate moment and recover all the clusters in those heavy bands. Intuitively, it takes much fewer samples for bands to stabilize than for individual clusters. For instance, consider a heavy band which consists of many small clusters. In this case, we will have seen most of the clusters in the band without too many samples, although many of the $\hat p_i$'s may be far off from $p_i$'s. The important observation here is that we have \emph{seen} those clusters; in contrast, if we consider individual clusters, we would have missed many of those clusters and will focus on recovering a few of them, which would cause a colossal waste of samples in the rejection sampling stage, for a similar reason explained at the beginning of this section. Therefore such banding approach enables us to control the number of samples needed. 

\begin{algorithm}[!ht]
\caption{The Improved Algorithm}\label{alg:main}
\begin{algorithmic}[1]
\State $I\gets \emptyset$ 
\State $k\gets 0$, $r\gets 0$, $K\gets 1/2$
\Repeat
	\State $K\gets 2K$
\Repeat
	\State $r\gets r+1$
	\State $Q\gets \emptyset$, $S\gets \emptyset$ 
	\State $T_1\gets 8\eps^{-1}\ln(10(k+1))$\tikzmark{ph1a} \label{alg:main T_1}
	\While{$Q=\emptyset$ and $|S|\leq T_1$}
		\State $(x,Q,S)\gets \textsc{Sample}(I,\{\tilde\mu_i\}_{i\in I},Q,S)$\tikzmark{right1}
	\EndWhile   \tikzmark{ph1b}
	\If {$Q \neq \emptyset$ and $k < K$}
		\Repeat 
			\State $(x,Q,S)\gets \textsc{Sample}(I,\{\tilde\mu_i\}_{i\in I},Q,S)\tikzmark{ph2a}$ 
			\For{each $i\in Q$}
				\State $\hat p_i \gets |\{u\in S: u = (x,i)\}|/|S|$
			\EndFor
			\State Split the clusters in to bands $\{B_\ell\}$ 
			\State $W\gets$ all clusters in heavy bands
		\Until {$|S| \!\geq\! 1600|W|\log|Q|\ln(10(k\!+\!|Q|))/\eps$}\tikzmark{ph2b}\tikzmark{right2}
		\State $q\gets |Q|$ \tikzmark{ph3a}
		\State $T_2 \gets 2^{17}|W|\log q\ln(10(k+q))/\eps^2$ \label{alg:main T_2}
		\While{$|S|\leq T_2$} 
			\State $(x,Q,S)\gets \textsc{Sample}(I,\{\tilde\mu_i\}_{i\in I},Q,S)$ 
		\EndWhile 
		\For{each $j\in W$}
			\State $x_j^\ast \gets \argmin_{x:(x,j)\in S} \Phi(\{x\}, \{\tilde\mu_i\}_{i\in I})$  \tikzmark{right3}
			\State $S_{j}\gets\emptyset$ 
		\EndFor \tikzmark{ph3b}
		\State $T_3 \gets 30K/\eps$  \tikzmark{ph4a} \label{alg:main T_3}
		\State $\{S_j\}_{j\in W} \gets \textsc{RejSamp}(I,\{\tilde\mu_i\}_{i\in I}, \quad \tikzmark{right4} \newline 
									\hspace*{10em}T_3,W,\{x_j^\ast\}_{j\in W})$
		\For{each $j\in W$}
			\State $\tilde\mu_j \gets (1/|S_j|)\sum_{x\in S_j} x$
		\EndFor \tikzmark{ph4b} 
		\State $k\gets k+|W|$
		\State $I\gets I\cup W$
	\EndIf
\Until{$Q=\emptyset$ or $k\geq K$}
\Until{$Q=\emptyset$ and $k\leq K$}
\AddNoteTight{ph1a}{ph1b}{right1}{Phase\\ 1}
\AddNoteTight{ph2a}{ph2b}{right2}{Phase\\ 2}
\AddNoteTight{ph3a}{ph3b}{right3}{Phase\\ 3}
\AddNoteTight{ph4a}{ph4b}{right4}{Phase\\ 4}
\end{algorithmic}
\end{algorithm}

Our improved algorithm is presented in Algorithm~\ref{alg:main}. Each round remains divided into four phases. Phase 1 remains testing new clusters. Starting from Phase 2 comes the change that instead of recovering only the largest sampled cluster $j$, we identify a subset $W$ of indices of the newly discovered clusters and recover all clusters in $W$. Phase 2 samples more points and identifies this subset $W$ by splitting the discovered clusters into bands and choosing all clusters in the heavy bands. Phase 3 samples more points and finds a reference point $x_j^\ast$ for each $j\in W$. Phase 4 keeps sample points until each cluster $j\in W$ sees $\Theta(\eps^{-1}|W|r)$ points and then calculates the approximate centers $\tilde\mu_j$ for each $j\in W$. The clusters in $W$ will then be added to $I$ as recovered clusters before the algorithm starts a new round.

A technical subtlety lies in controlling the failure probability in our new algorithm. In the basic algorithm, the failure probability in the $r$-th round is $a_r = 1/(r\poly(\log r))$. If we recover $w_r$ clusters in $r$-th round, this failure probability for each of the $w_r$ clusters needs to be $1/(w_r a_r)$, and as a consequence, $\tilde O(w_r^2 a_r)$ samples are needed in the $r$-th round. In the worst case this becomes $\tilde O(K^3)$ samples, the same as in the basic algorithm, when both $w_r = \Theta(K)$ and $a_r = \tilde{\Theta}(1/K)$. To resolve this, we guess $K=1,2,\dots$ in powers of $2$ and assign $a_r$ to be $w_r/K$. For each fixed guess $K$, the number of samples is bounded by $\tilde{O}(K^2)$; iterating over guesses incurs only an additional $\log K$ factor.

The analysis of our improved algorithm follows the same sketch of the basic algorithm and we only present the changes below. All proofs are postponed to Appendix~\ref{sec:improve proofs}. The next lemma offers a similar guarantee as Lemma~\ref{lem:correctness basic}.

\begin{lemma}\label{lem:correctness}
Upon the termination of Algorithm~\ref{alg:main}, with probability at least $0.95$, the $\eps$-reducibility condition w.r.t.\ $I$ is satisfied and an approximate centroid $\tilde\mu_i$ that satisfies~\eqref{eqn:good_approx_centroid} is obtained for every $i\in I$.
\end{lemma}

Next we upper bound the number of samples with a lemma analogous to Lemma~\ref{lem:each p_i basic}.
%

\begin{lemma}\label{lem:each p_i}
With probability at least $1-1/(50(k+q)^2)$, it holds that $p_j\geq 1/(70|B_{\ell(j)}|\log q)$ for all $j\in W$.
\end{lemma}

Recall that $Y_j$ was defined in \eqref{eqn:Y_j_def} for all $j\in W$.
%
The following two lemmata are the analogue of Lemmata~\ref{lem:Y_j basic} and \ref{lem:rejection sampling basic}, respectively.

\begin{lemma}\label{lem:Y_j}
Assume that the event in Lemma~\ref{lem:each p_i} holds.
With probability at least $1-1/(25(k+q)^2)$, it holds that $x_j^\ast\in Y_j$ for all $j\in W$.
\end{lemma}

\begin{lemma}\label{lem:rejection}
Assume that the event in Lemma~\ref{lem:Y_j} holds.
By sampling $s = 2^{28}\eps^{-4} |W| K \log q$ points in total, with probability at least $1-|W|\exp(-8K)$, every cluster $j\in W$ has $T_3 = 30K/\eps$ samples returned by the rejection sampling procedure.
\end{lemma}

Now we are ready to prove the main theorem of the improved algorithm.

\begin{theorem}\label{thm:main}
With probability at least $0.9$, Algorithm~\ref{alg:main} (the improved algorithm) finds all the clusters and obtain for every $i$ an approximate centroid $\tilde\mu_i$ that satisfies~\eqref{eqn:good_approx_centroid}, using $O(\eps^{-4} K L^2 \log^3 K\log L)$ same-cluster queries in total, where $K$ is the number of recovered clusters and $L$ is the total number of discovered clusters.
\end{theorem}

\begin{remark}
The preceding theorem improves the query complexity of the basic algorithm by about a factor of $K$, as desired. In particular, when the clusters are $\eps$-irreducible with respect to any subcollection of the clusters, we shall recover all $L$ clusters using $O(\eps^{-4}L^3\log^3 L)$ same-cluster queries, better than basic algorithm by about a factor of $L$.
\end{remark}

\section{Noisy Oracles}
Our algorithm can be extended to the case where the same-cluster oracle errs with a constant probability $p<1/2$. The full details are postponed to Appendix~\ref{sec:noisy oracle} in the full version.

\begin{theorem}
Suppose that the same-cluster oracle errs with a constant probability $p<1/2$. With probability at least $0.6$, there is an algorithm that finds all the clusters and obtains for every $i$ an approximate centroid $\tilde\mu_i$ that satisfies~\eqref{eqn:good_approx_centroid}, using $\tilde{O}(\eps^{-6}L^6K)$ same-cluster queries in total.
\end{theorem}

%% file: exp.tex
\section{Experiments}
\label{sec:exp}

In this section, we conduct experimental studies of our algorithms. As we mentioned before, all previous algorithms need to know $K$ and it is impossible to convert them to the case of unknown $K$ (with the only exception of~\cite{CPM18}). In particular, we note the inapplicability of the $k$-means algorithm and the algorithm in~\cite{ABJK18}.
\begin{itemize}[leftmargin=15pt]
	\item The $k$-means algorithm is for an unsupervised learning task and returns clusters which are always spherical, so it is undesirable for arbitrary shapes of clusters and it makes sense to examine the accuracy in terms of misclassified points. Our problem is semi-supervised with a same-cluster oracle, which can be used to recover clusters of arbitrary shapes and guarantees no misclassification. The assessment of the algorithm is therefore the number of discovered clusters and the number of samples.  Owing to the very different nature of the problems, the $k$-means algorithm should not be used in our problem or compared with algorithms designed specifically for our problem.
	\item The algorithm in~\cite{ABJK18} runs in $k$ rounds and take $2^{12}k^3/\eps^2$ samples in the first step of each round. With $k=10$ and $\eps=0.1$, it needs to sample in each round at least $4\times 10^8$ points, usually larger than the size of a dataset. (Our algorithm can be easily simplified to handle such cases, see the subsection titled ``Algorithms'' below.)
\end{itemize}
The only exception, the algorithm in~\cite{CPM18}, is just simple uniform sampling and whether or not $K$ is known is not critical. This uniform sampling algorithm will be referred to as \uniform\ below. 


\medskip \noindent\textbf{Datasets.} \  We test our algorithms using both synthetic and real world datasets. 

For synthetic datasets, we generate $n$ points that belong to $K$ clusters in the $d$-dimensional Euclidean space. Since in the real world the cluster sizes typically follow a {\em power-law} distribution, we assign to each cluster a random size drawn from the widely used {\em Zipf distribution}, which has the density function $f(x)\propto x^{-\alpha}$, where $x$ is the size of the cluster and $\alpha$ a parameter controlling the skewness of the dataset. We then choose a center $\mu_i$ for each cluster $i\in [K]$ in a manner to be specified later and generate the points in the cluster from the multivariate Gaussian distribution $N(\mu_i, \sigma^2 I_d)$, where $I_d$ denotes the $d \times d$ identity matrix.

Now we specify how to choose the centers. In practice, clusters in the dataset are not always well separated; there could be clusters whose centers are close to each other.  We thus use an additional parameter $p$ to characterize this phenomenon.  In the default setting of $p = 0$, all centers of the $K$ clusters are drawn uniformly at random from $[0, b]^d$.  When $p > 0$, we first partition the clusters into groups as follows: Think of each cluster as a node. For each pair of clusters, with probability $p$ we add an edge between the two nodes.  Each connected component of the resulting graph forms a group.  Next, for each group of clusters, we pick a random cluster and choose its center $\mu$ uniformly at random in $[0, b]^d$. For each of the remaining clusters in the group, we choose its center uniformly at random in the neighborhood of radius $\rho$ centered at $\mu$.

We use the following set of parameters as the default setting in our synthetic datasets:
$n=10^6$, $K=100$, $\alpha=2.5$, $\sigma=0.3$, $b=5$, $d=10$ and $\rho=0.1$.

We also use the following two real-world datasets.
\begin{itemize}[leftmargin=20pt]
\item
  \shuttle~\footnote{\url{https://archive.ics.uci.edu/ml/datasets/Statlog+(Shuttle)}.}: it describes the radiator positions in a NASA space shuttle.
There are 58,000 points with 9 numerical attributes, and 7 clusters in total. 

\item
  \kdd~\footnote{\url{http://kdd.ics.uci.edu/databases/kddcup99/kddcup99.html}.}: this dataset is taken from the 1999 KDD Cup contest. It contains about 4.9M points classified into 23 clusters.  The original dataset has 41 attributes. We retain all numerical attributes except one that contains only zeros, resulting in 33 attributes in total. 
\end{itemize}
Each feature in the real world datasets is normalized to have zero mean and unit standard deviation.

\smallskip

\noindent{\bf Algorithms.} \  We compare three algorithms listed below. A cluster is said to be {\em heavy} when the number of (uniform) samples obtained from this cluster is more than a predetermined threshold, which is set to be $10$ in our experiments.
\begin{itemize}[leftmargin=15pt]
\item \uniform: This is uniform sampling.  That is, we keep getting random samples from the dataset one by one and identify their label by same-cluster queries. We recover a cluster when it becomes heavy. 

\item \baseline: This is based on our basic algorithm (Algorithm~\ref{alg:basic}).  We recover clusters one at a time when it becomes heavy.

\item \improved: This is a simplified version of the improved algorithm (Algorithm~\ref{alg:main}). Instead of partitioning clusters to bands and then recovering all clusters in the heavy bands, in each round we just keep sampling points until the fraction of points in the heavy clusters is more than a half, at which moment we try to recover all heavy clusters in a batch.  
\end{itemize}

Recall that in our basic and improved algorithms, we always ``ignore'' the samples belonging to the unrecovered clusters in the previous rounds, which will not affect our theoretical analysis.  But in practice it is reasonable to reuse these ``old'' samples in the succeeding rounds so that we are able to recover some clusters earlier.  
Because of such a sample-reuse procedure, the practical performance difference between the \baseline\ and \improved\ algorithms becomes less significant. Recall that in theory, the main improvement of \improved\ over \baseline\ is that we try to avoid wasting too many samples in each round by recovering possibly more than one cluster at a time.  But still, as we shall see shortly, \improved\ outperforms \baseline\ in all metrics, though sometimes the gaps may not seem significant.  


\smallskip

\noindent{\bf Measurements.} \
To measure the same-cluster query complexity, we introduce two assessments.  The first is {\em fixed-budget}, where given a fixed number of same-cluster queries, we compare the number of clusters recovered by different algorithms.  The second is {\em fixed-recovery}, where each algorithm needs to recover a predetermined number of clusters, and we compare the numbers of same-cluster queries the algorithms use.

We also report the error of the approximate centroid of each recovered cluster.  For a recovered cluster $X_i$, let $\mu_i$ be its centroid and $\hat\mu_i$ be the approximate center.  We define the centroid approximation error to be $\left(\Phi(X_i, \hat\mu_i) - \Phi(X_i, \mu_i)\right) / \Phi(X_i, \mu_i)$.

For \baseline\ and \improved, we further compare their running time and round usage.  We note that a small round usage is very useful 
if we want to efficiently parallelize the learning process. 

Finally, we would like to mention one subtlety when measuring the number of same-cluster queries.  Since all algorithms that we are going to test are sampling based, and for each sampled point we need to identify its label by same-cluster queries, which, in the worst case, takes $k$ queries, where $k$ is the number of the discovered clusters so far. In practice, however, a good query order/strategy can save us a lot of same-cluster queries.  We apply the following query strategy for all tested algorithms.  

\begin{heuristic}
\label{heu:classify}
For each discovered cluster, we maintain its approximate center (using the samples we have obtained).  When a new point is sampled, we query the centers of discovered clusters based on their distances to the new point in the {\em increasing} order as follows. Suppose that the new point is $x$. We calculate the distances between $x$ and all approximate centers $\hat\mu_i$, and sort them in increasing order, say, $d(x,\hat\mu_1) \leq d(x,\hat\mu_2)\leq \cdots$. Then for $i=1,2,\dots$ sequentially, we check whether $x$ belongs to cluster $i$ by querying whether $x$ and some sample point from cluster $i$ belong to the same cluster. If $x$ does not belong to any discovered cluster, a new cluster will be created.
\end{heuristic}

We will also use this heuristic to test the effectiveness of approximate centers for classifying new points. That is, for a newly inserted point $q$, we count the number of same-cluster queries needed to correctly classify $q$ using the approximate centers and Heuristic~\ref{heu:classify}.

\smallskip

\noindent{\bf Computation Environments.} \
All algorithms are implemented in C++, and all experiments were conducted on Dell PowerEdge T630 server with two Intel Xeon E5-2667 v4 3.2GHz CPU with eight cores each and 256GB memory. 

\subsection{Results}
\label{sec:result}
The three algorithms, \uniform, \baseline\ and \improved, are compared on the same-cluster query complexity, the quality of returned approximate centers, the running time and the number of rounds. 
All results are the average values of $100$ runs of experiments. 


\smallskip

\input{query_complexity_figures}

\noindent{\bf Query Complexity.}  For each of the synthetic and real-world datasets, we measure the query complexities under both fixed-budget and fixed-recovery settings. 

The results for synthetic dataset, the \shuttle\ dataset and the \kdd\ dataset are plotted in Figure~\ref{fig:synthetic},  Figure~\ref{fig:shuttle} and Figure~\ref{fig:kdd}, respectively. In all figures, the left column corresponds to the fixed-budget case and the right column the fixed-recovery case. We note that some points for \uniform\ are missing since they are out of
the boundary. The exact values of all experiments can be found in Appendix~\ref{sec:query_complexity_tables} of the full version.

For all datasets, \baseline\ and \improved\ significantly outperform \uniform. Take \kdd\ dataset for example, we observe that with $10^5$ query budget \uniform\ recovers $10$ clusters while \baseline\ and \improved\ recover around $18$ clusters.  In order to recover $16$ clusters, \uniform\ requires $200,000$ queries per cluster, while \baseline\ and \improved\ only requires no more than $2,500$ queries per cluster.  Even on small datasets such as \shuttle, \uniform\ needs about twice amount of queries to recover all clusters compared with \baseline\ and \improved.  We also observe that \improved\ performs slightly better than \baseline.

On synthetic datasets, one can see that higher collision probability $p$ makes clustering more difficult for all algorithms.  Given a query budget of $2 \times 10^5$, \baseline\ and \improved\ can recover about $70$ clusters when $p = 0$ (no collision),  but only about $60$ clusters when $p = 0.3$. 

We also observe that \improved\ performs better than \baseline\ in the \kdd\ dataset in all settings, except the fixed-recovery setting for the \shuttle\ dataset, in which they have almost the same performance.

\smallskip

\noindent{\bf Center Approximation.} \ We compare the approximation error of each algorithm to show the quality of returned approximate centers. We run all three algorithms in the fixed-recovery setting because we can only measure the error of the recovered clusters. We examine the median error among all recovered clusters and observe that the centroid approximation errors are indeed similar on all datasets when the same number of clusters is recovered. Detailed results are listed in Tables~\ref{tab:sim-error} to \ref{tab:kdd-error}. We observe that all centroid approximation error rates are below 10\%: the error rates are about 9\% on the synthetic datasets and are about 5\% on the two real-world datasets. There is no significant difference between the error rates of different algorithms; the maximum difference is no more than 2\%. These consistent results indicate that the approximate centers computed by all algorithms are of good quality. 

\input{centroid_tables}

\smallskip


\noindent{\bf Time and Rounds.} \ We compare the running time and the number of
rounds between \baseline\ and \improved, with results presented in
Figure~\ref{fig:sim-round} to Figure~\ref{fig:kdd-time}.
Here we present the running time in the sampling stage, showing that the reduced
round usage in \improved\ also leads to a reduced usage of sampling time in \improved.
In terms of the time saving in querying stage, we have already discussed it in the
query complexity part.
The only exception is the \shuttle\ dataset, where the performance of \baseline\ and \improved\ are almost the same; this may be because the number of clusters in \shuttle\ is too small ($7$ in total) for the algorithms to make any difference. We remark that fewer rounds implies a shorter waiting time for synchronization in the parallel learning setting and a smaller communication cost in the distributed learning setting.

\input{runtime_figures}

\smallskip

\noindent{\bf Classification Using Approximate Centers.} \
We now measure the quality of the approximate centers outputted by our algorithms for point classification using Heuristic~\ref{heu:classify}.  That is, after the algorithm terminates and outputs the set of approximate centers, we try to cluster new (unclassified) points using Heuristic~\ref{heu:classify} and count the number of same-cluster queries used for each new point.  Since there is only a very small number of points in the dataset that have been sampled/classified during the run of \baseline\ and \improved, we simply use the whole dataset as test points. This should not introduce much bias to our measurement.

Our results for the synthetic and the real-world datasets are plotted in Figures~\ref{fig:sim-querynum} and~\ref{fig:real-querynum}, respectively.  We observe that for all three tested algorithms, for the majority of the points in the dataset, only {\em one} same-cluster query is needed for classification; in other words, we only need to find their nearest approximate centers for determining their cluster labels.  Except for the \shuttle\ dataset for which about 10-15\% points need at least three same-cluster queries, in all other datasets the fraction of points which need at least three queries is negligible.  Overall, the performance of the three tested algorithms is comparable.

\input{query_num_figs}

\smallskip

\noindent{\bf Summary.} \ Briefly summarizing our experimental results, we observe that both \baseline\ and \improved\ significantly outperform \uniform\ in terms of query complexities.  All three algorithms have similar center approximation quality. \improved\ outperforms \baseline\ by a large margin in both running time and round usage on the synthetic and \kdd\ datasets.

%% file: query_complexity_figures.tex
\begin{figure}
\begin{minipage}[d]{0.49\linewidth}
\begin{overpic}[width=\textwidth,percent]{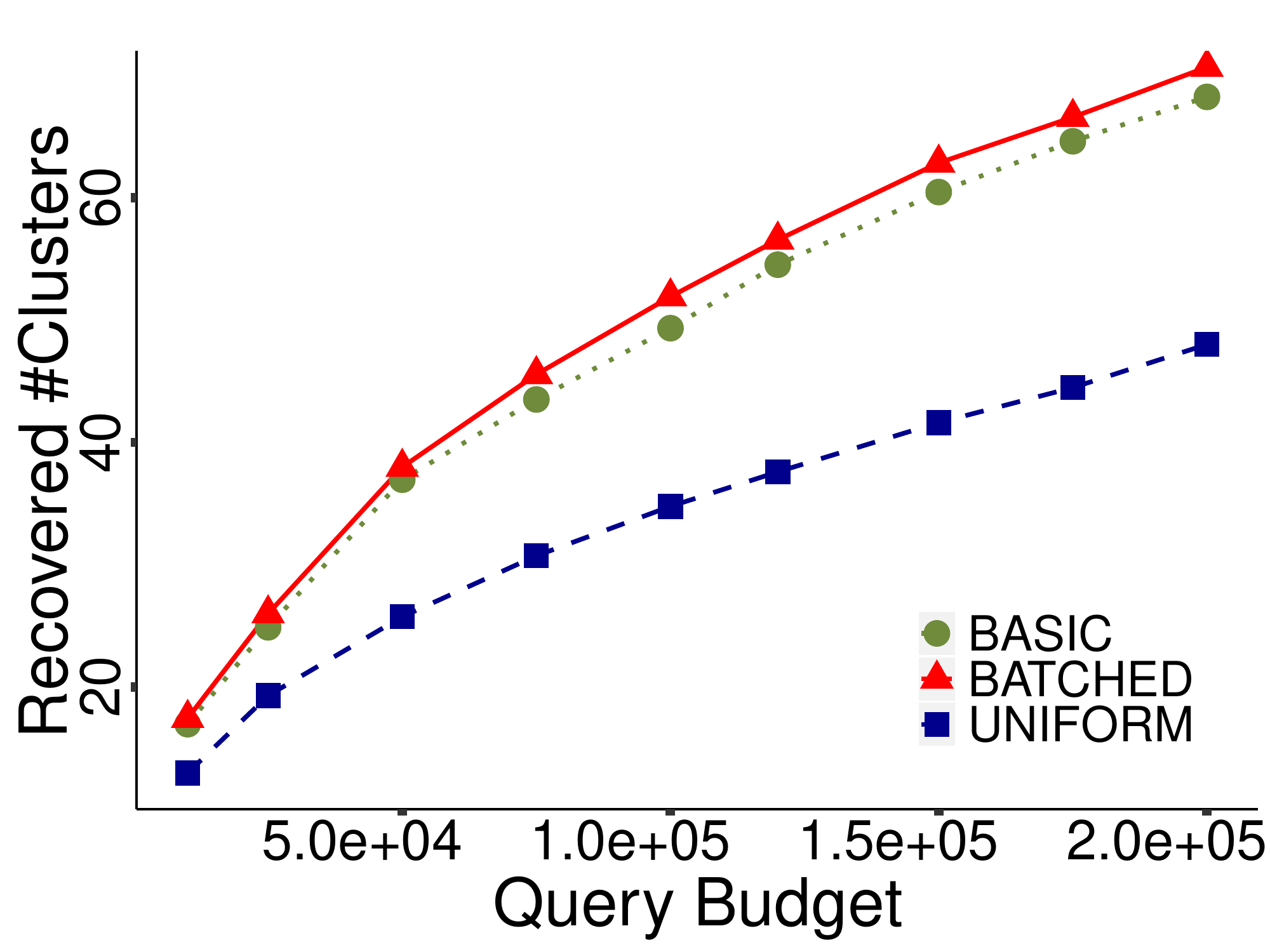}
	\put(13,65){\scriptsize $p=0$}
\end{overpic}
\end{minipage}
\hfill
\begin{minipage}[d]{0.49\linewidth}
\begin{overpic}[width=\textwidth,percent]{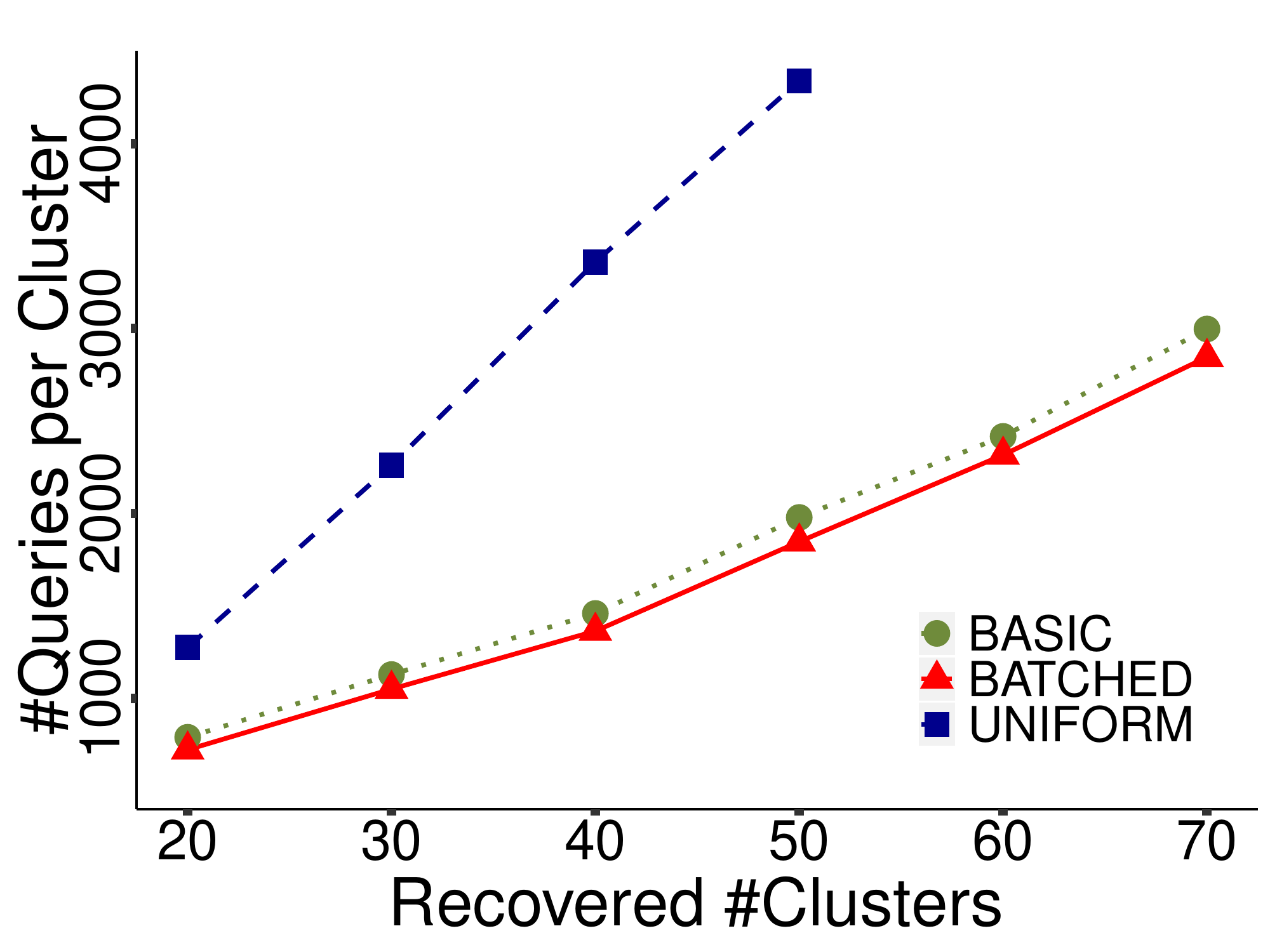}
	\put(13,65){\scriptsize $p=0$}
\end{overpic}
\end{minipage}

\begin{minipage}[d]{0.49\linewidth}
\begin{overpic}[width=\textwidth,percent]{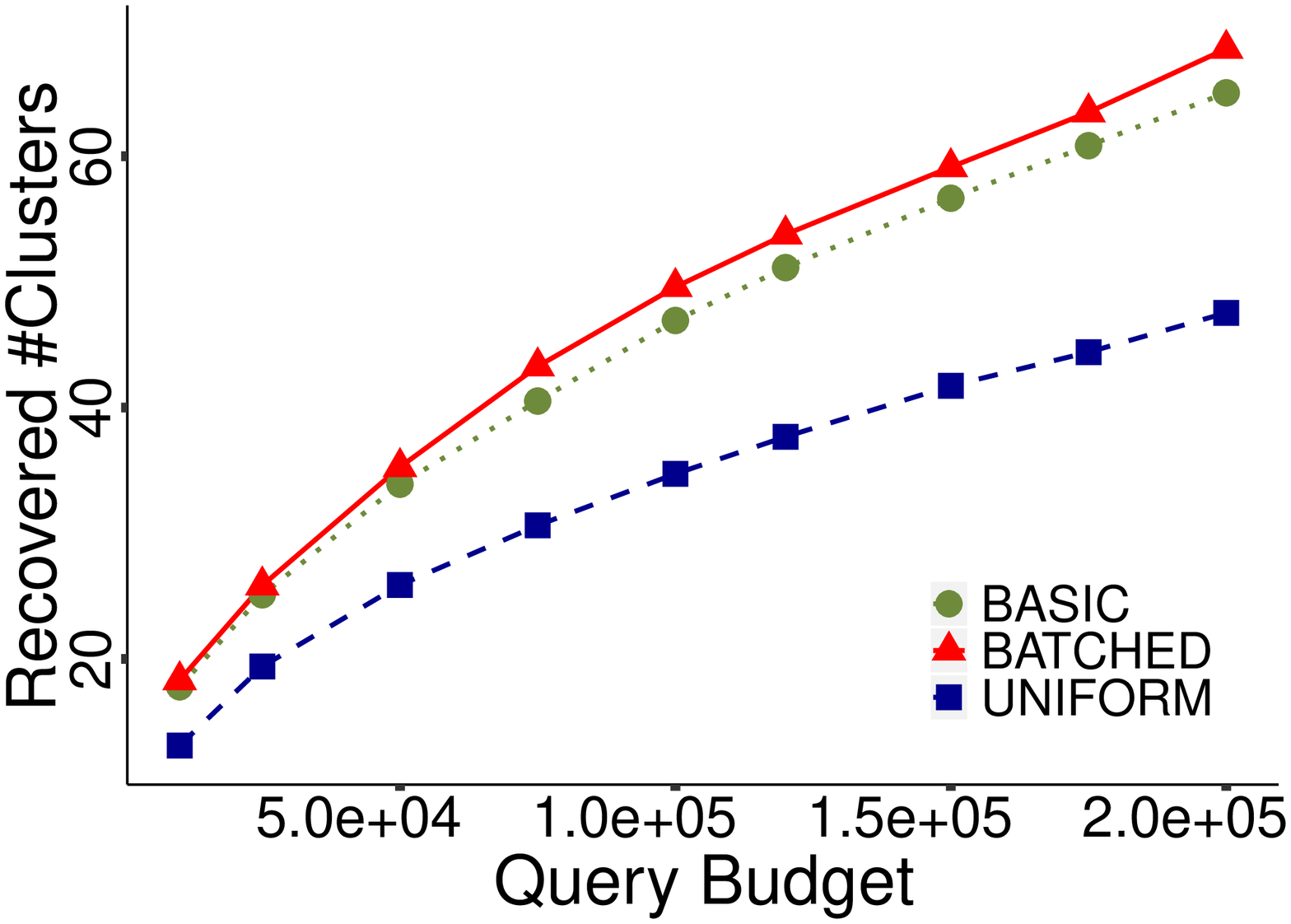}
	\put(13,63){\scriptsize $p=0.1$}
\end{overpic}
\end{minipage}
\hfill
\begin{minipage}[d]{0.49\linewidth}
\begin{overpic}[width=\textwidth,percent]{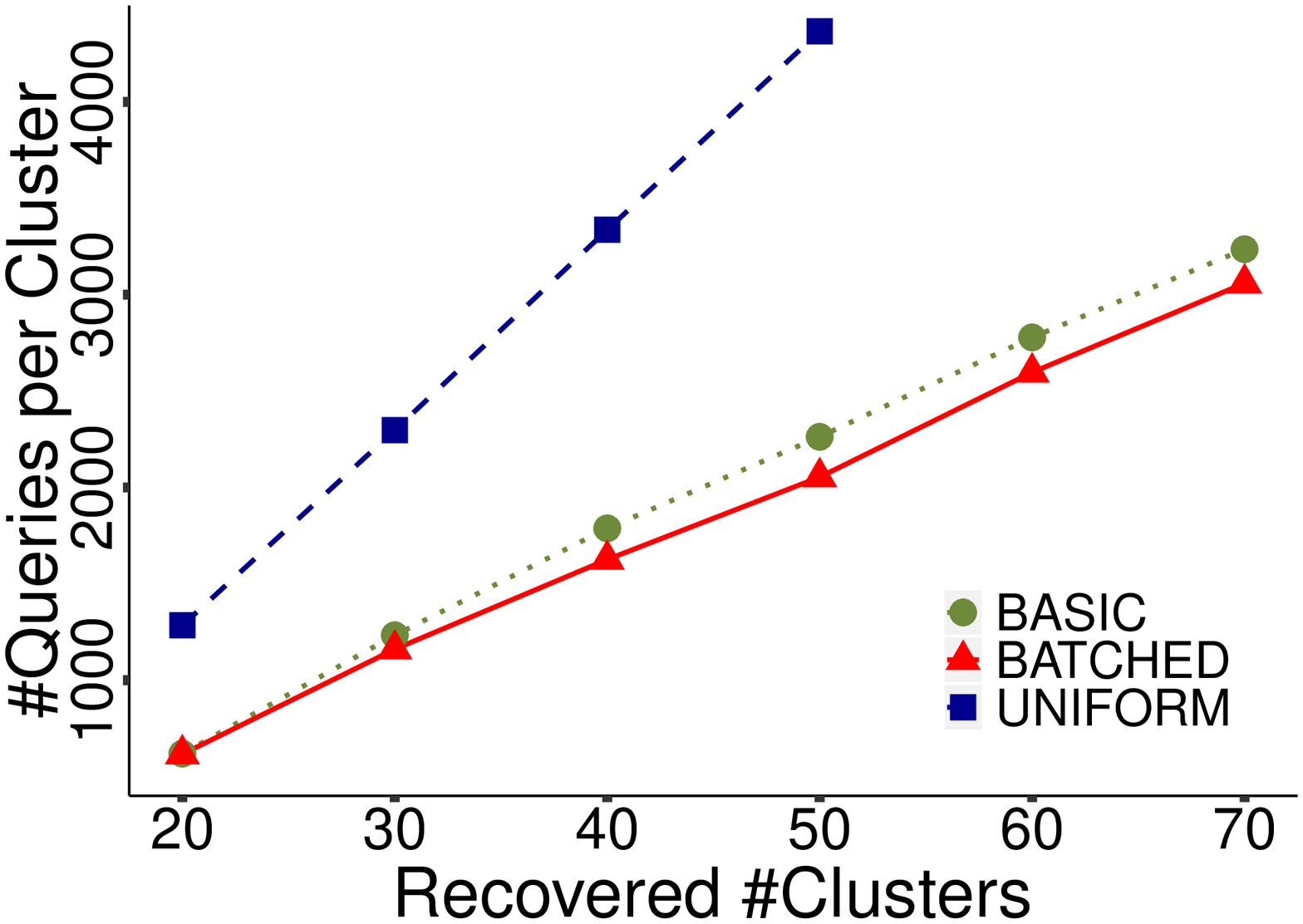}
	\put(13,63){\scriptsize $p=0.1$}
\end{overpic}
\end{minipage}

\begin{minipage}[d]{0.49\linewidth}
\begin{overpic}[width=\textwidth,percent]{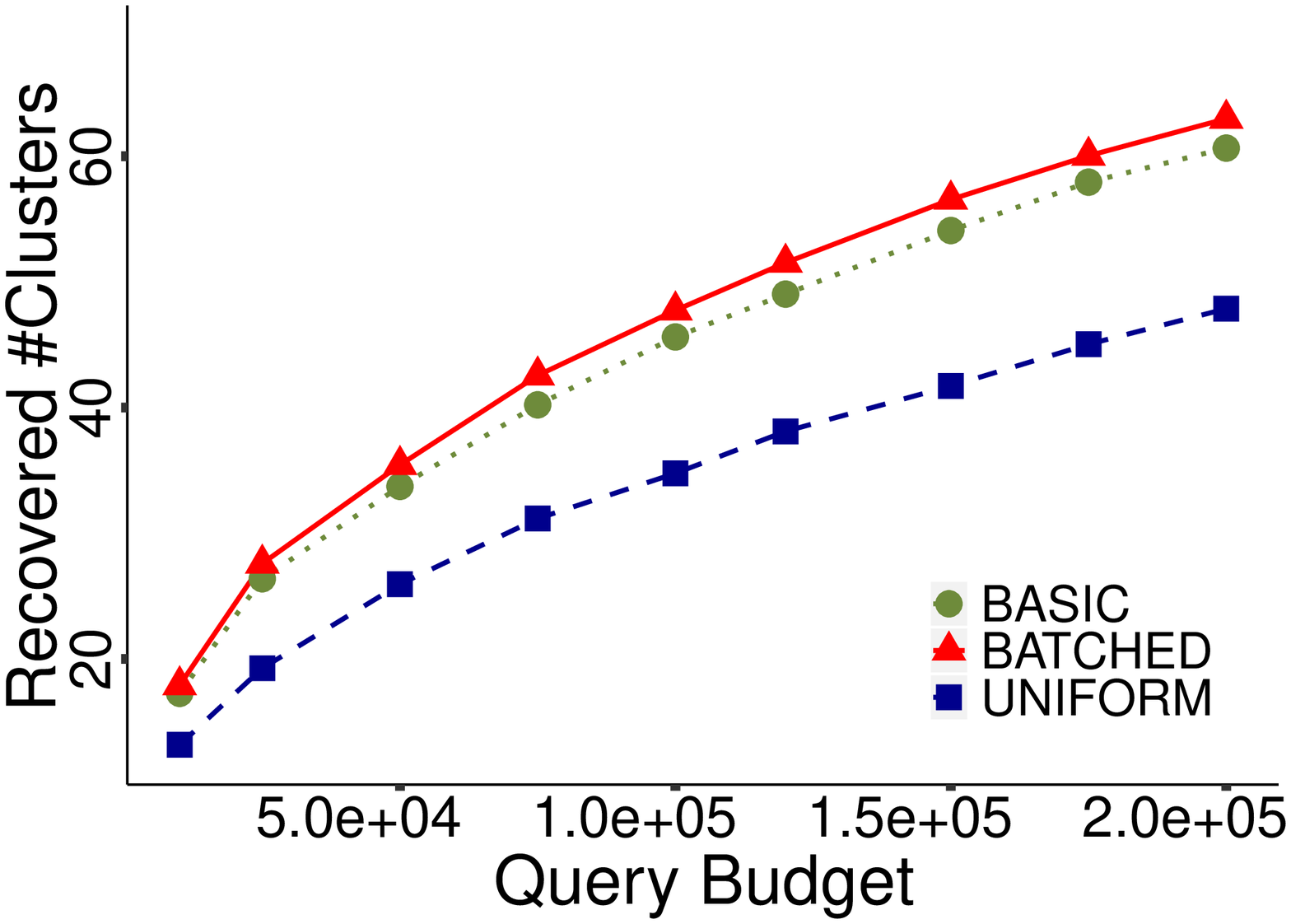}
	\put(13,63){\scriptsize $p=0.3$}
\end{overpic}
\centerline{Fixed-budget}
\end{minipage}
\hfill
\begin{minipage}[d]{0.49\linewidth}
\begin{overpic}[width=\textwidth,percent]{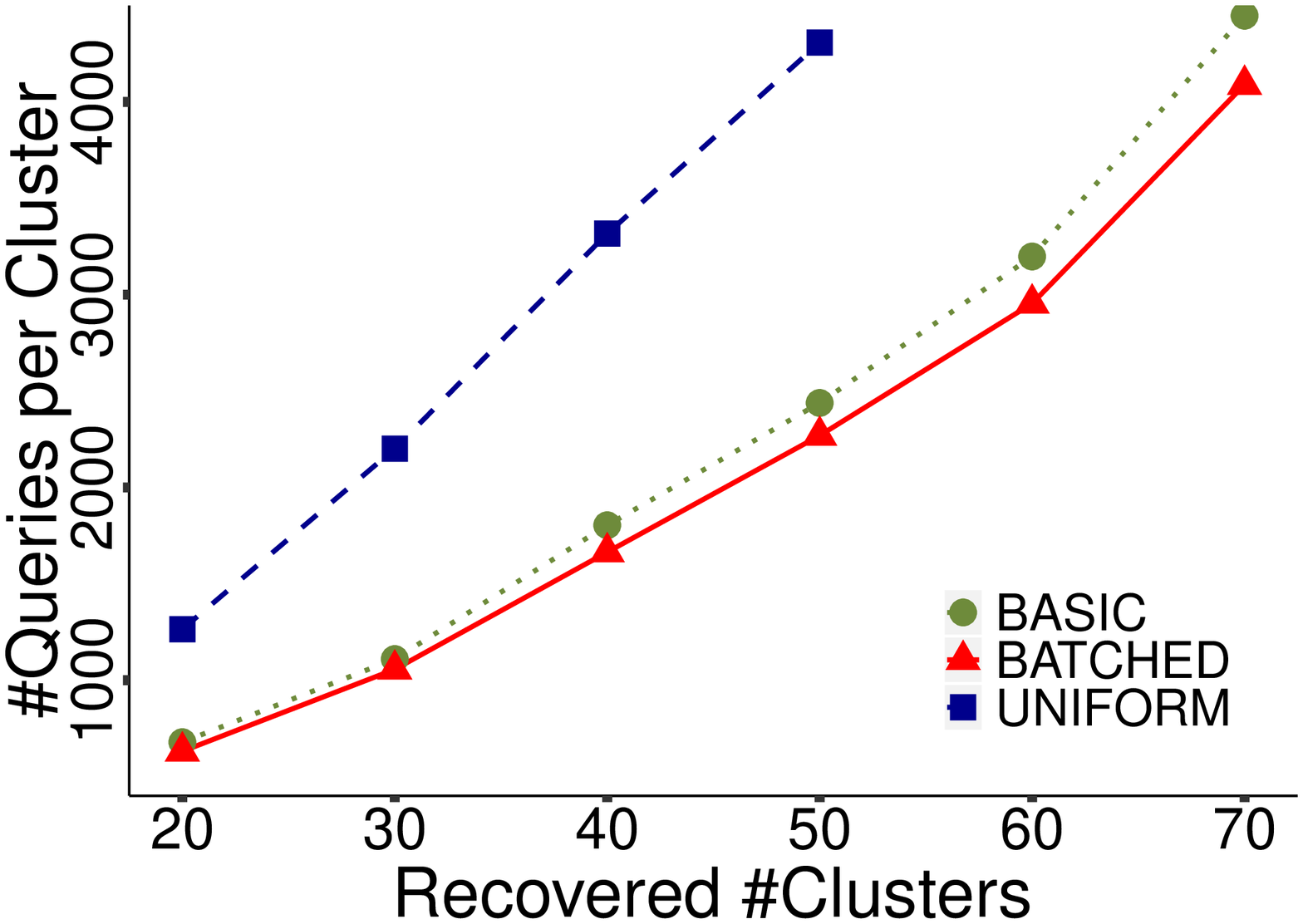}
	\put(13,63){\scriptsize $p=0.3$}
\end{overpic}
\centerline{Fixed-recovery}
\end{minipage}
\vspace{-0.5em}
\caption{Performance comparison on synthetic datasets.}
\label{fig:synthetic}
\end{figure}

\begin{figure}
\centering
\begin{minipage}[d]{0.49\linewidth}
\includegraphics[width=\textwidth]{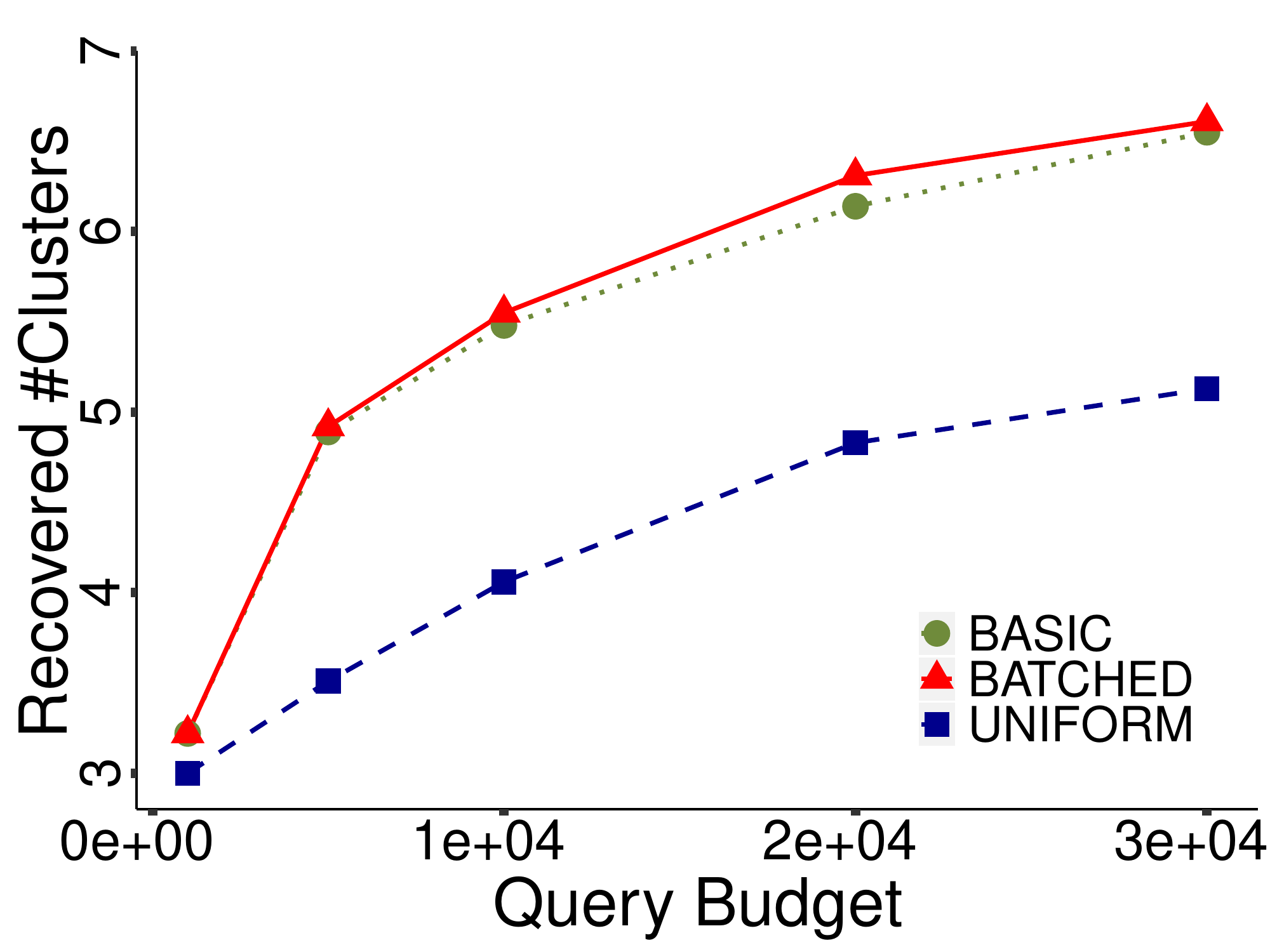}
\centerline{Fixed-budget}
\end{minipage}
\hfill
\begin{minipage}[d]{0.49\linewidth}
\centering
\includegraphics[width=\textwidth]{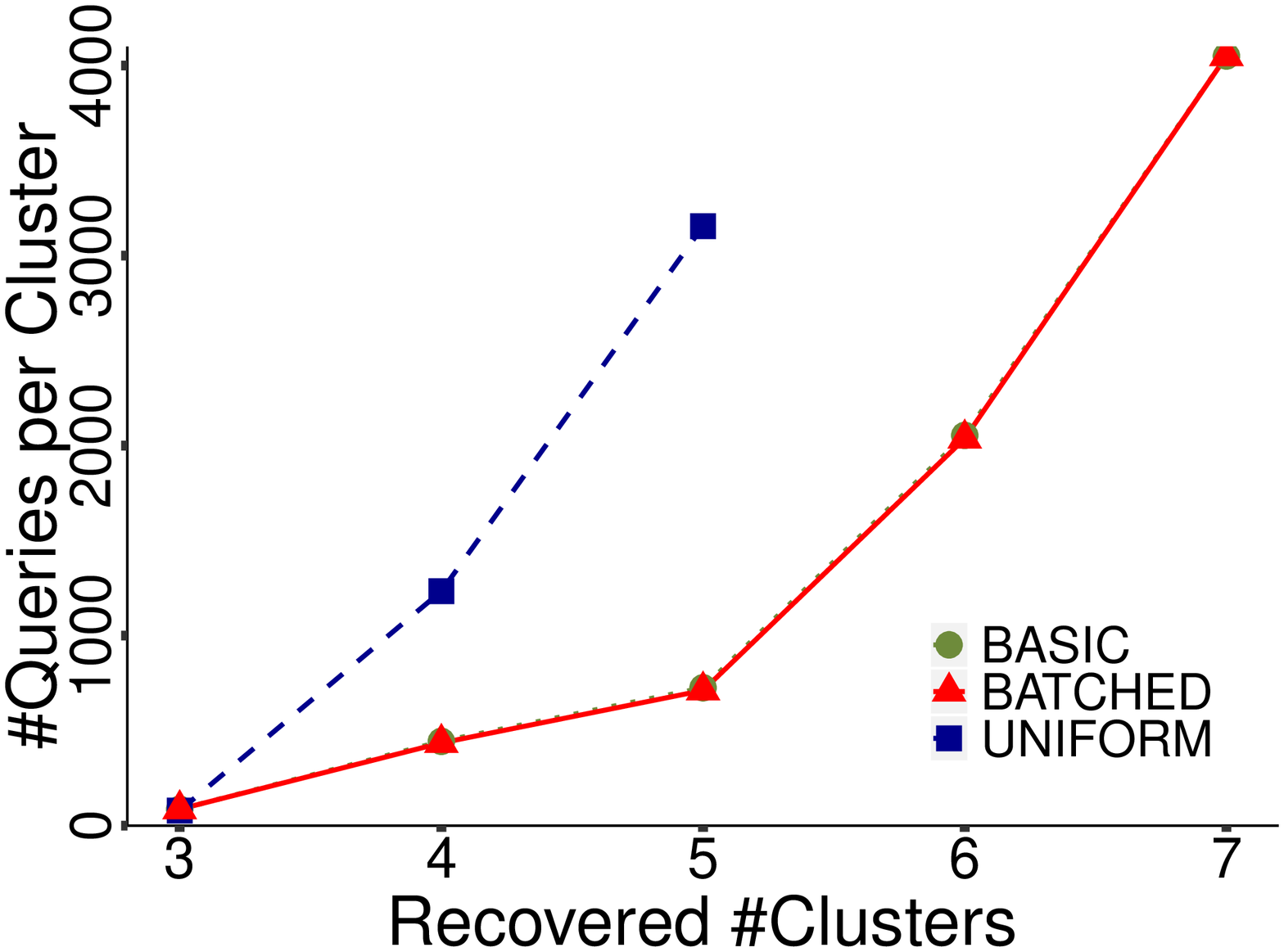}
\centerline{Fixed-recovery}
\end{minipage}
\vspace{-0.7em}
\caption{Performance comparison on \shuttle\ dataset. The curves for \baseline\ and \improved\ almost overlap for fixed-recovery.}
\label{fig:shuttle}
\end{figure}
\hfill
\begin{figure}
\begin{minipage}[d]{0.49\linewidth}
\includegraphics[width=\textwidth]{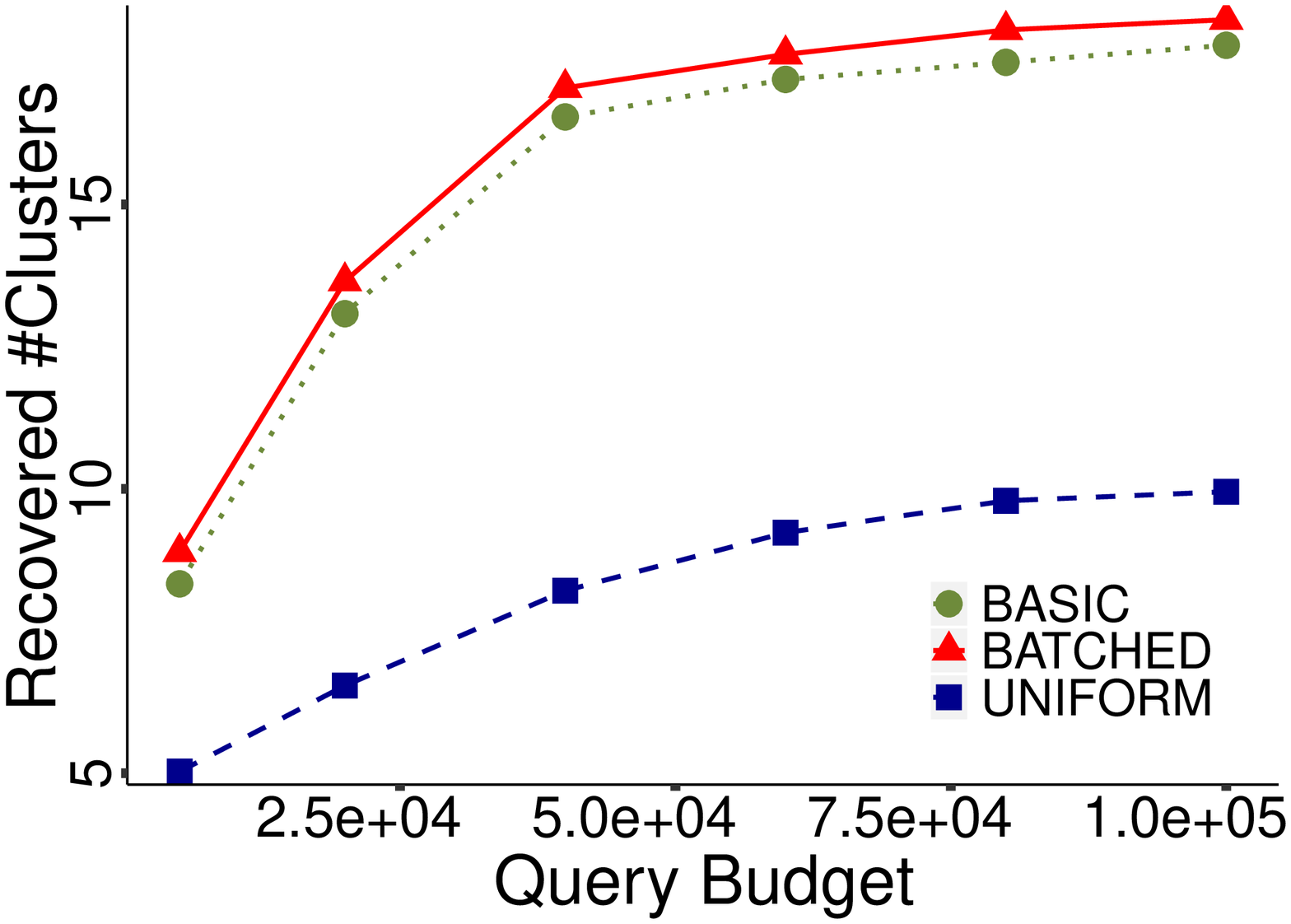}
\centerline{Fixed-budget}
\end{minipage}
\hfill
\begin{minipage}[d]{0.49\linewidth}
\includegraphics[width=\textwidth]{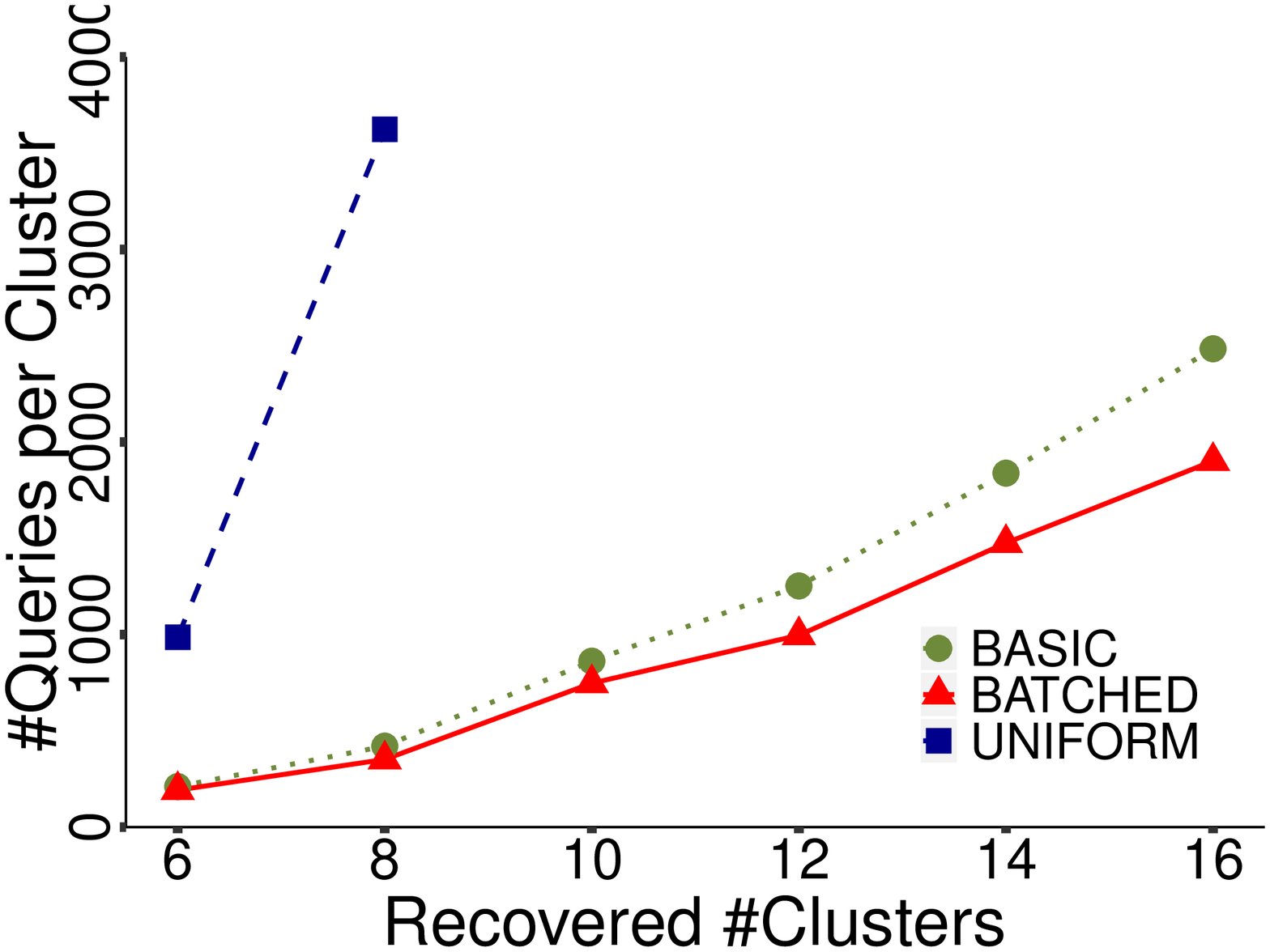}
\centerline{Fixed-recovery}
\end{minipage}
\vspace{-0.7em}
\caption{Performance comparison on \kdd\ dataset.}
\label{fig:kdd}
\end{figure}

%% file: centroid_tables.tex
\begin{table}
\setlength\tabcolsep{1.5pt}
\def\arraystretch{0.9}
\centering
\begin{tabular}{|l|l|r|r|r|r|r|r|}
\hline
\#Clusters   & \multicolumn{1}{c|}{Algorithms} & \multicolumn{1}{c|}{20} & \multicolumn{1}{c|}{30} & \multicolumn{1}{c|}{40} & \multicolumn{1}{c|}{50} & \multicolumn{1}{c|}{60} & \multicolumn{1}{c|}{70} \\ \hline
\multirow{3}{*}{$p=0$} & \baseline\ & 8.70\%  & 8.70\%  & 8.50\%  & 8.41\% & 8.43\% & 8.37\% \\
            & \improved\  & 8.74\% & 8.60\%  & 8.56\% & 8.64\% & 8.43\% & 8.51\% \\
            & \uniform\  & 9.90\%  & 9.43\% & 9.58\% & 9.22\% & 9.24\% & 9.18\% \\ \hline
\multirow{3}{*}{$p=0.1$} & \baseline\ & 8.86\% & 8.68\% & 8.72\% & 8.39\% & 8.51\% & 8.45\% \\
            & \improved\  & 8.76\% & 8.64\% & 8.45\% & 8.49\% & 8.50\%  & 8.42\% \\
            & \uniform\  & 9.59\% & 9.64\% & 9.39\% & 9.42\% & 9.30\%  & 9.18\% \\ \hline
\multirow{3}{*}{$p=0.3$}   & \baseline\ & 8.87\% & 8.68\% & 8.56\% & 8.40\%  & 8.34\% & 8.38\% \\
            & \improved\  & 8.95\% & 8.64\% & 8.57\% & 8.37\% & 8.35\% & 8.39\% \\
            & \uniform\  & 9.75\% & 9.44\% & 9.45\% & 9.35\% & 9.25\% & 9.28\% \\ \hline
\end{tabular}
\caption{Error rates on synthetic datasets}
\label{tab:sim-error}
%
\begin{tabular}{|l|r|r|r|r|r|r|}
\hline
\#Clusters   & \multicolumn{1}{c|}{2} & \multicolumn{1}{c|}{3} & \multicolumn{1}{c|}{4} & \multicolumn{1}{c|}{5} & \multicolumn{1}{c|}{6} & \multicolumn{1}{c|}{7} \\ \hline
  \baseline\ &9.72\% & 4.99\% & 6.28\% & 5.40\%  & 5.85\% & 5.91\% \\
  \improved\ &8.94\% & 4.78\% & 6.33\% & 5.44\% & 5.89\% & 5.66\% \\ 
  \uniform\ &9.32\% & 4.92\% & 5.79\% & 5.67\% & 5.65\% & 5.50\% \\ \hline
\end{tabular}
\caption{Error rates on \shuttle\ dataset}
\label{tab:shuttle-error}
%
\begin{tabular}{|l|r|r|r|r|r|r|}
\hline
\#Clusters   & \multicolumn{1}{c|}{6} & \multicolumn{1}{c|}{8} & \multicolumn{1}{c|}{10} & \multicolumn{1}{c|}{12} & \multicolumn{1}{c|}{14} & \multicolumn{1}{c|}{16} \\ \hline
\baseline\ & 3.70\%  & 3.83\% & 4.87\% & 5.75\% & 5.95\% & 5.68\% \\
\improved\ & 3.57\% & 3.68\% & 4.49\% & 4.39\% & 5.04\% & 5.53\% \\
\uniform\ & 3.96\% & 5.21\% & 4.92\% & 4.83\% & 4.57\% & 5.43\% \\ \hline
\end{tabular}
\caption{Error rates on \kdd\ dataset}
\label{tab:kdd-error}
\vspace{-3mm}
\end{table}

%% file: runtime_figures.tex
\begin{figure}
\begin{minipage}[d]{0.49\linewidth}
\begin{overpic}[width=\textwidth,percent]{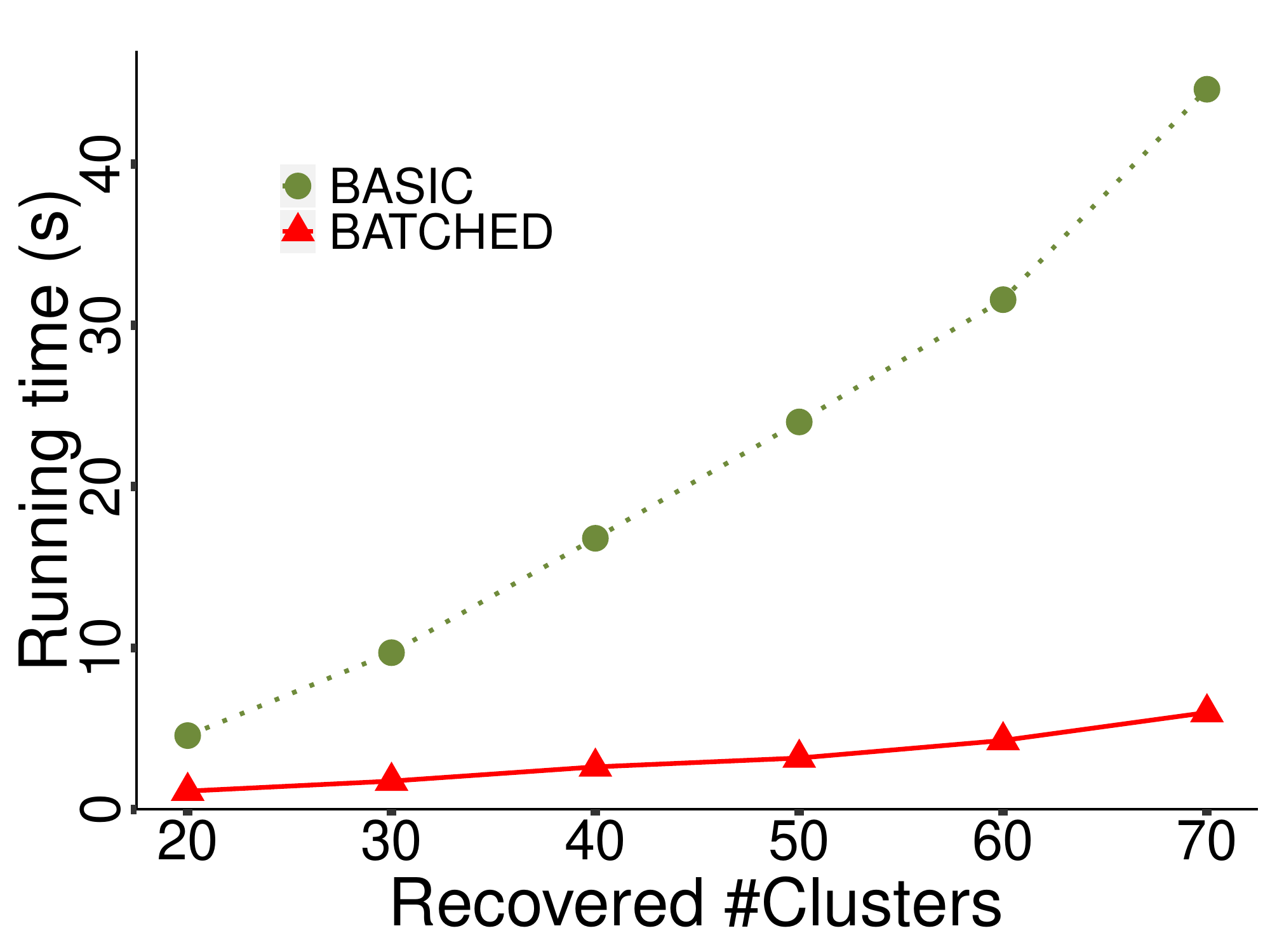}
	\put(13,65){\scriptsize $p=0$}
\end{overpic}
\end{minipage}
\hfill
\begin{minipage}[d]{0.49\linewidth}
\begin{overpic}[width=\textwidth,percent]{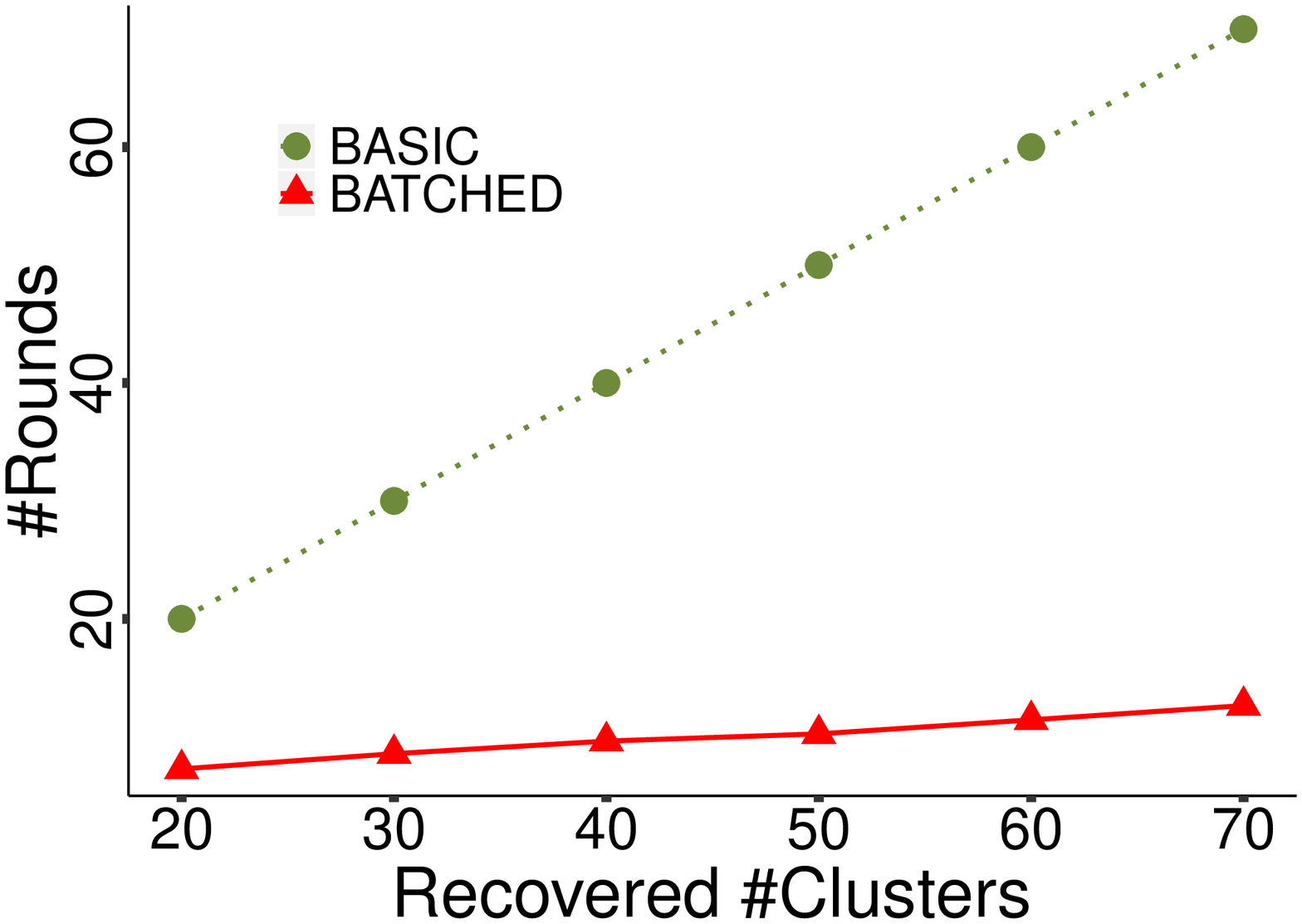}
	\put(13,65){\scriptsize $p=0$}
\end{overpic}
\end{minipage}

\begin{minipage}[d]{0.49\linewidth}
\begin{overpic}[width=\textwidth,percent]{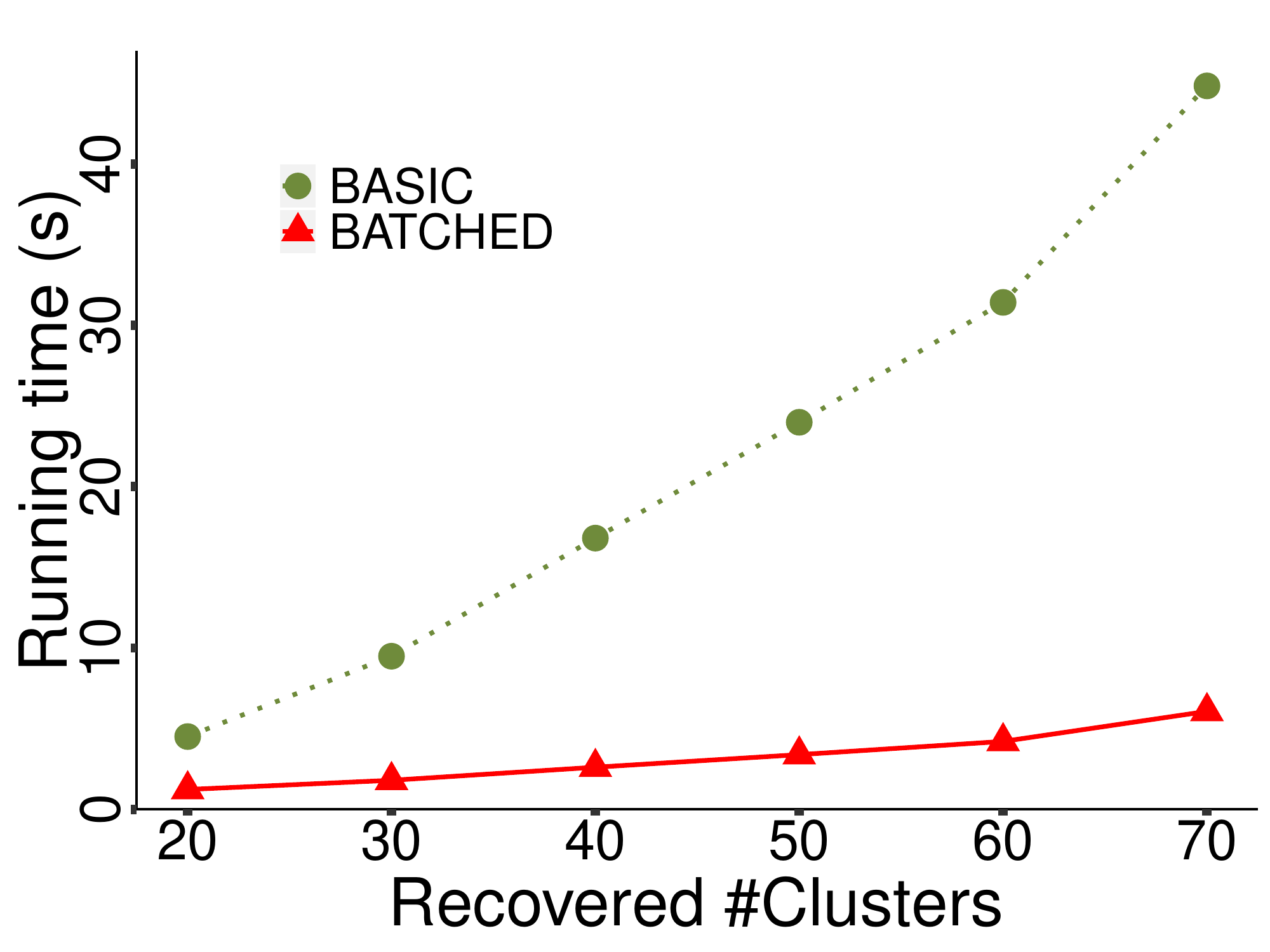}
	\put(13,65){\scriptsize $p=0.1$}
\end{overpic}
\end{minipage}
\hfill
\begin{minipage}[d]{0.49\linewidth}
\begin{overpic}[width=\textwidth,percent]{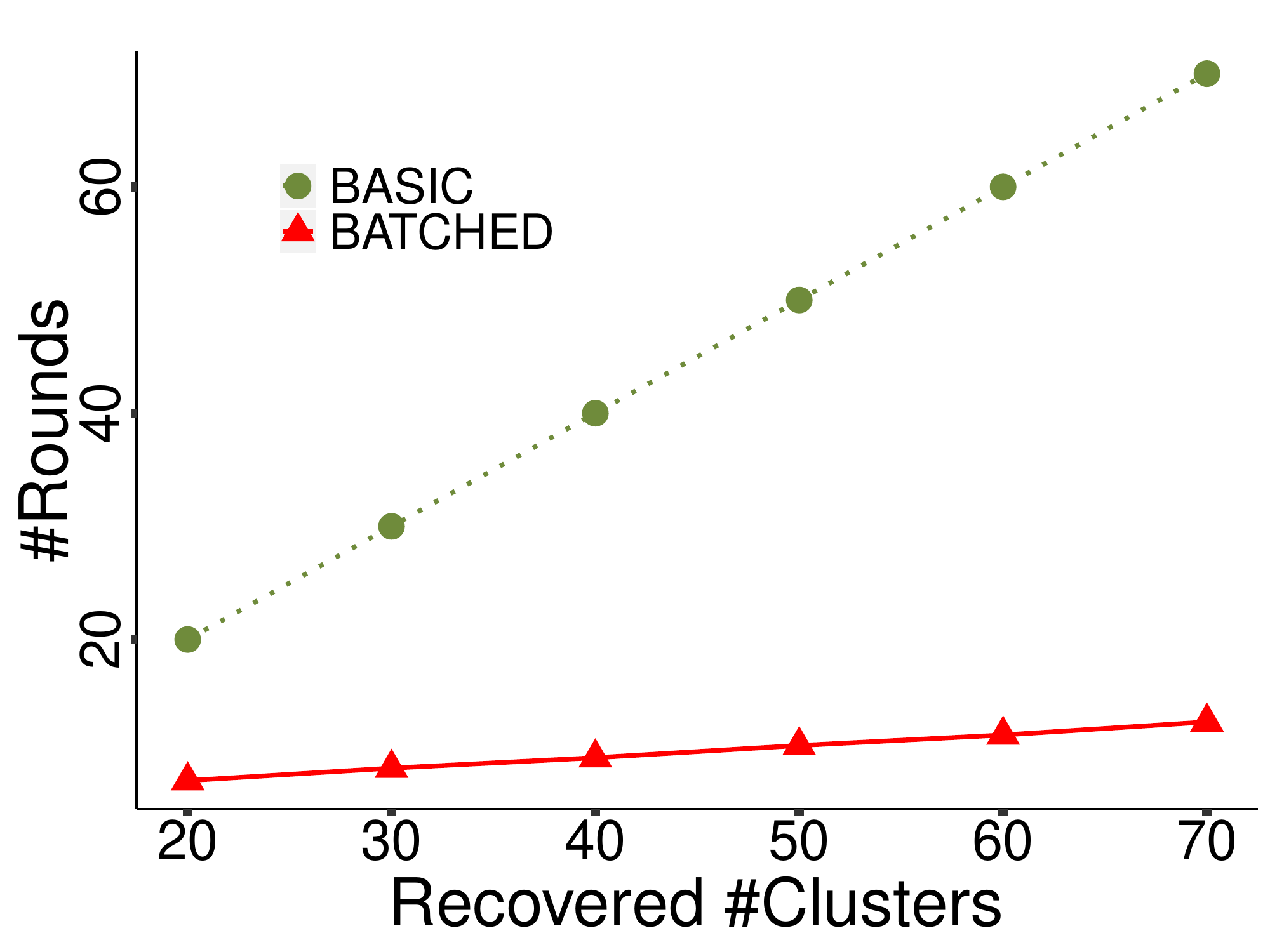}
	\put(13,65){\scriptsize $p=0.1$}
\end{overpic}
\end{minipage}

\begin{minipage}[d]{0.49\linewidth}
\begin{overpic}[width=\textwidth,percent]{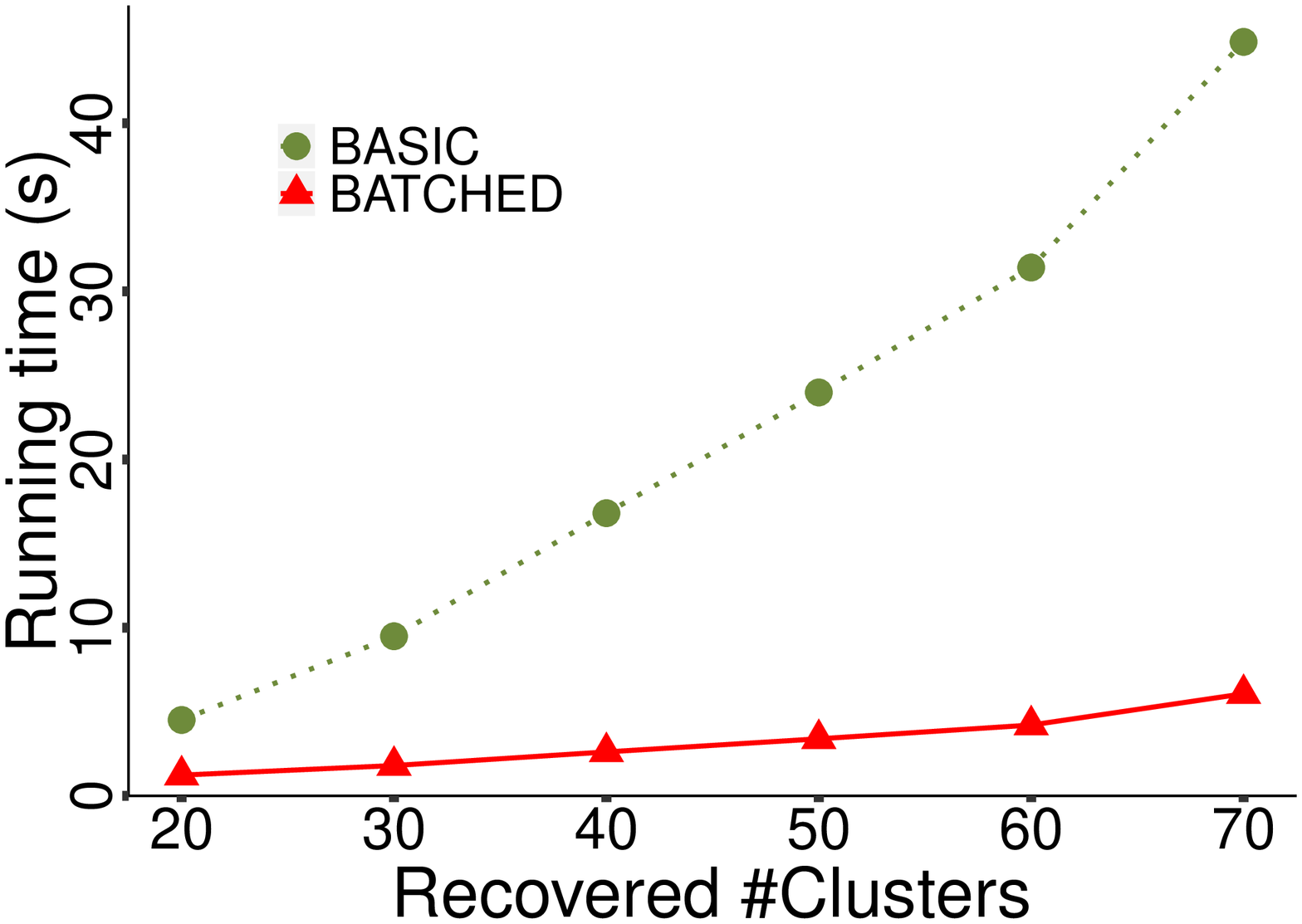}
	\put(13,65){\scriptsize $p=0.3$}
\end{overpic}
\end{minipage}
\hfill
\begin{minipage}[d]{0.49\linewidth}
\begin{overpic}[width=\textwidth,percent]{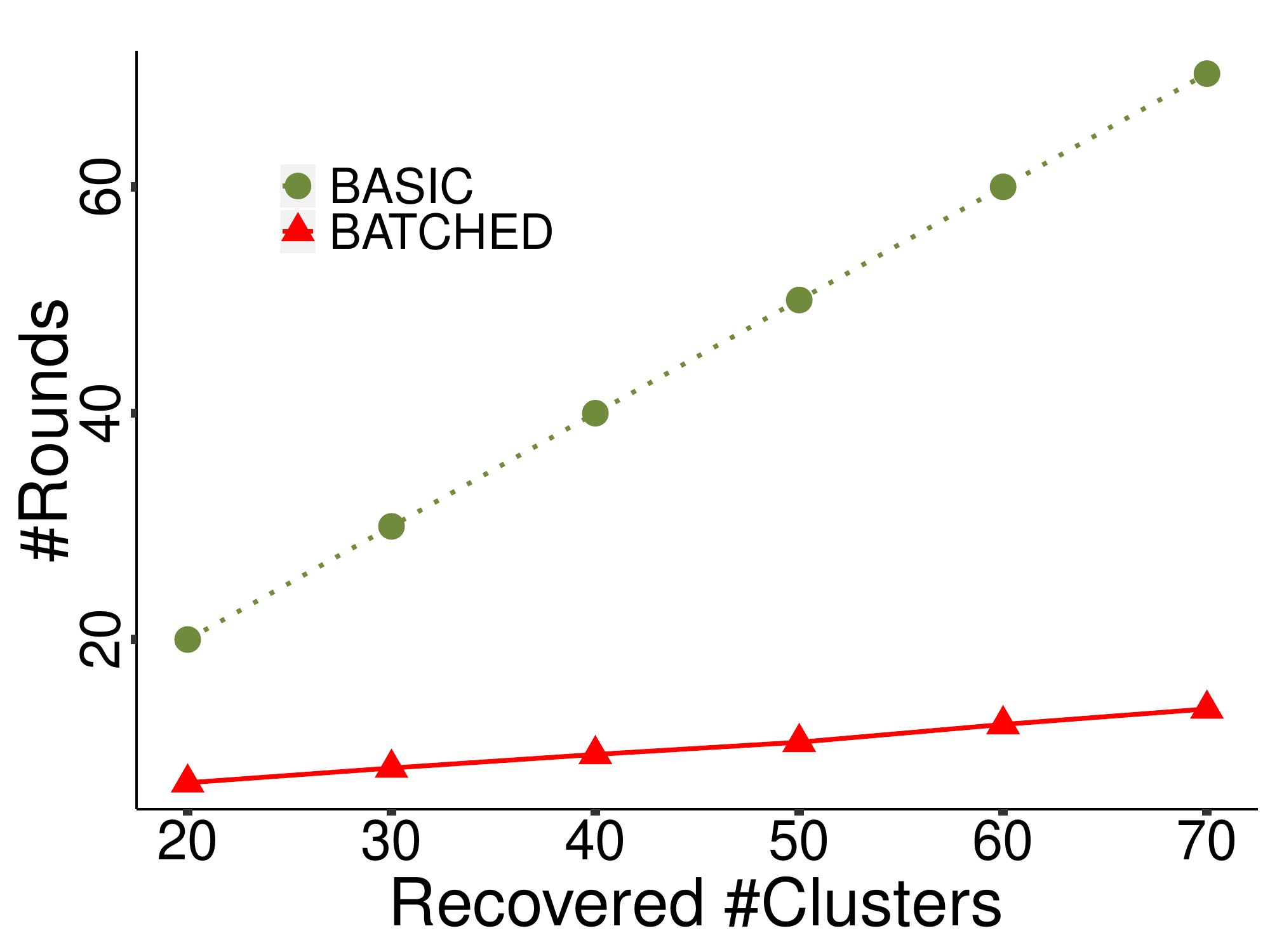}
	\put(13,65){\scriptsize $p=0.3$}
\end{overpic}
\end{minipage}
\vspace{-0.7em}
\caption{Running time and round usage on synthetic datasets}
\label{fig:sim-round}
\end{figure}

\begin{figure}
\centering
\begin{minipage}[d]{0.49\linewidth}
\centering
\includegraphics[width=\textwidth]{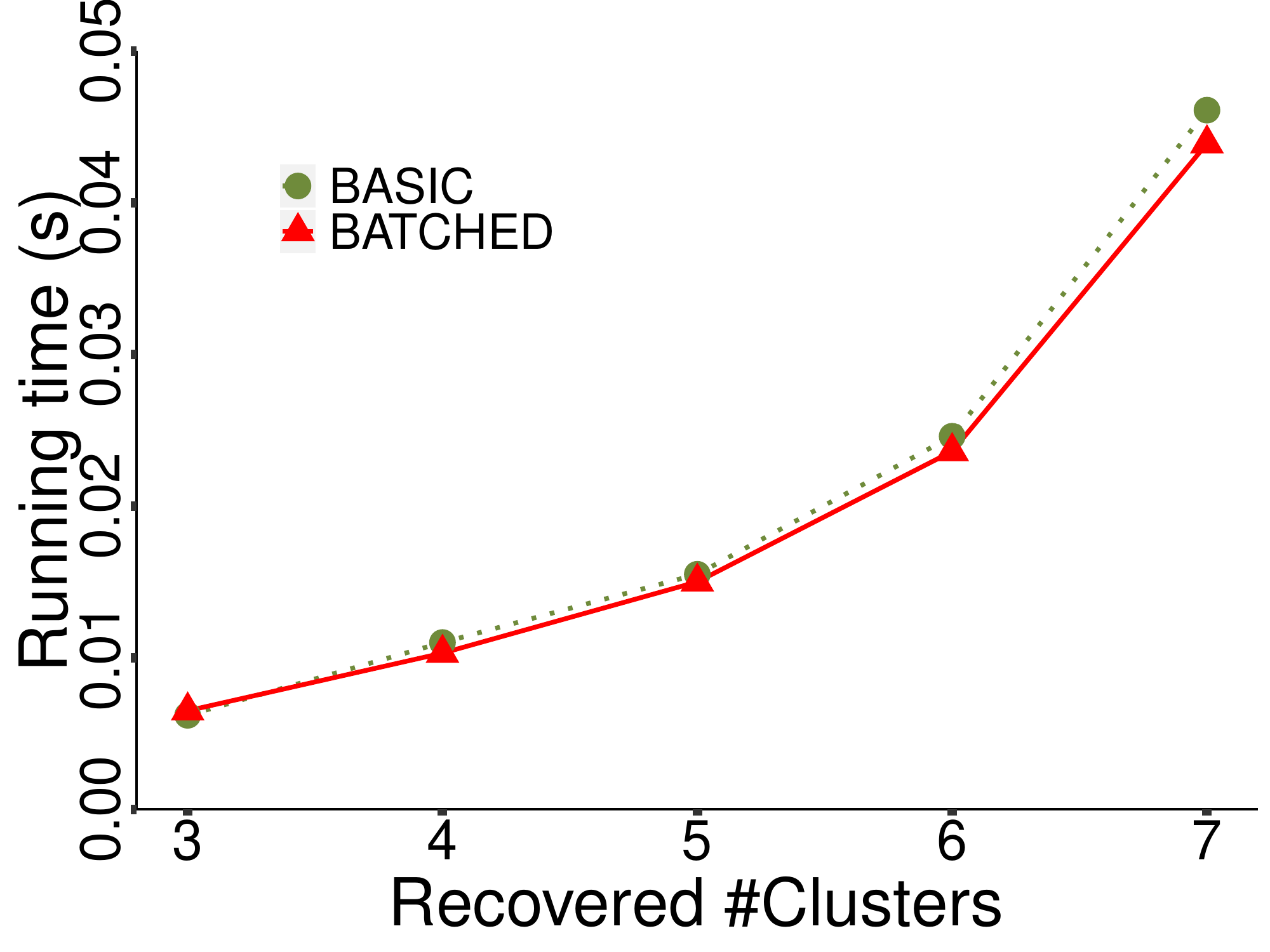}
\end{minipage}
\hfill
\begin{minipage}[d]{0.49\linewidth}
\centering
\includegraphics[width=\textwidth]{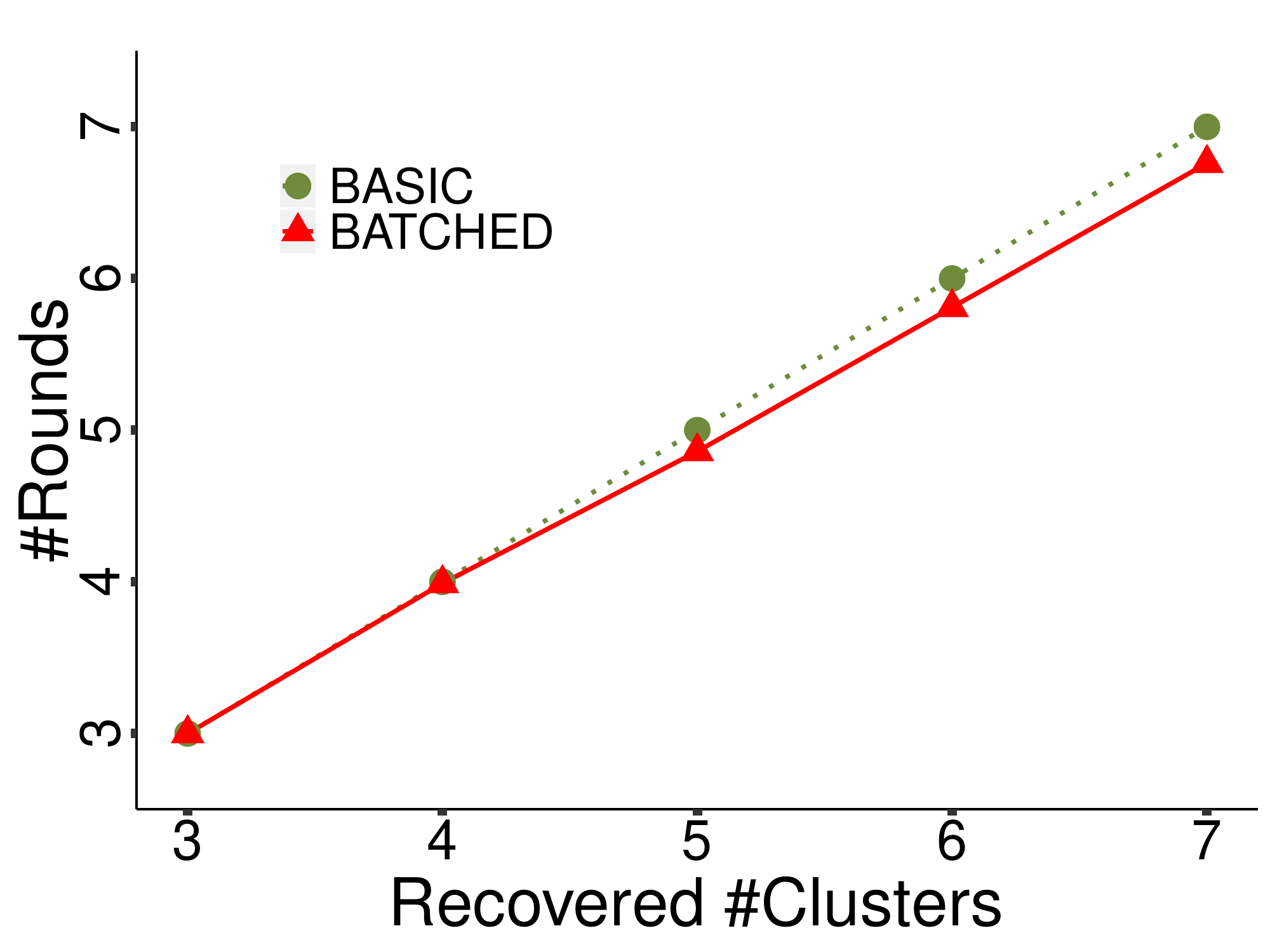}
\end{minipage}
\vspace{-0.7em}
\caption{Running time and round usage on \shuttle\ dataset}
\label{fig:shuttle-time}
\end{figure}

\begin{figure}
\centering
\begin{minipage}[d]{0.49\linewidth}
\centering
\includegraphics[width=\textwidth]{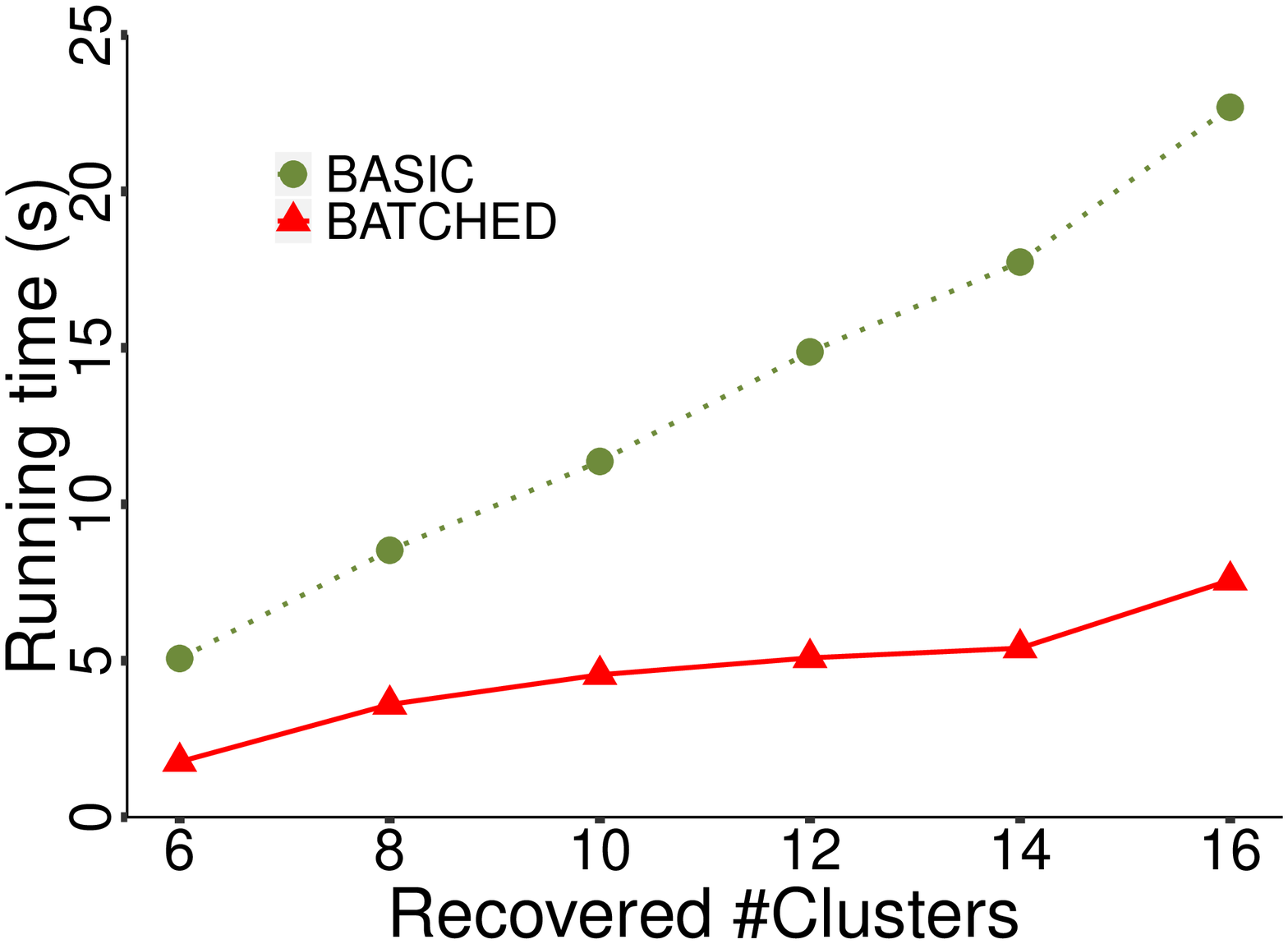}
\end{minipage}
\begin{minipage}[d]{0.49\linewidth}
\centering
\includegraphics[width=\textwidth]{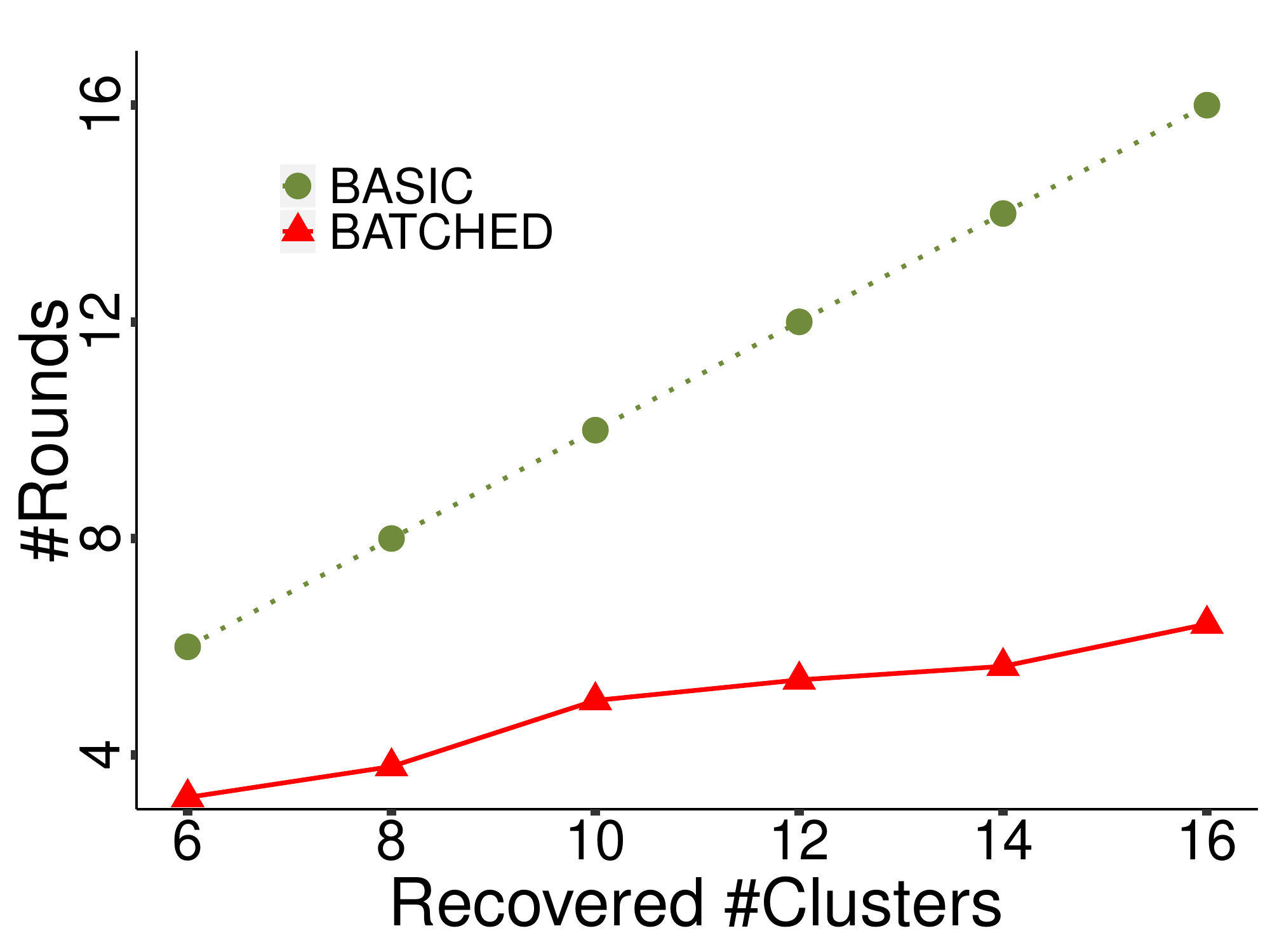}
\end{minipage}
\vspace{-0.7em}
\caption{Running time and round usage on \kdd\ dataset}
\label{fig:kdd-time}
\end{figure}

%% file: query_num_figs.tex
\begin{figure}
\begin{minipage}[d]{0.31\linewidth}
\includegraphics[width=\textwidth]{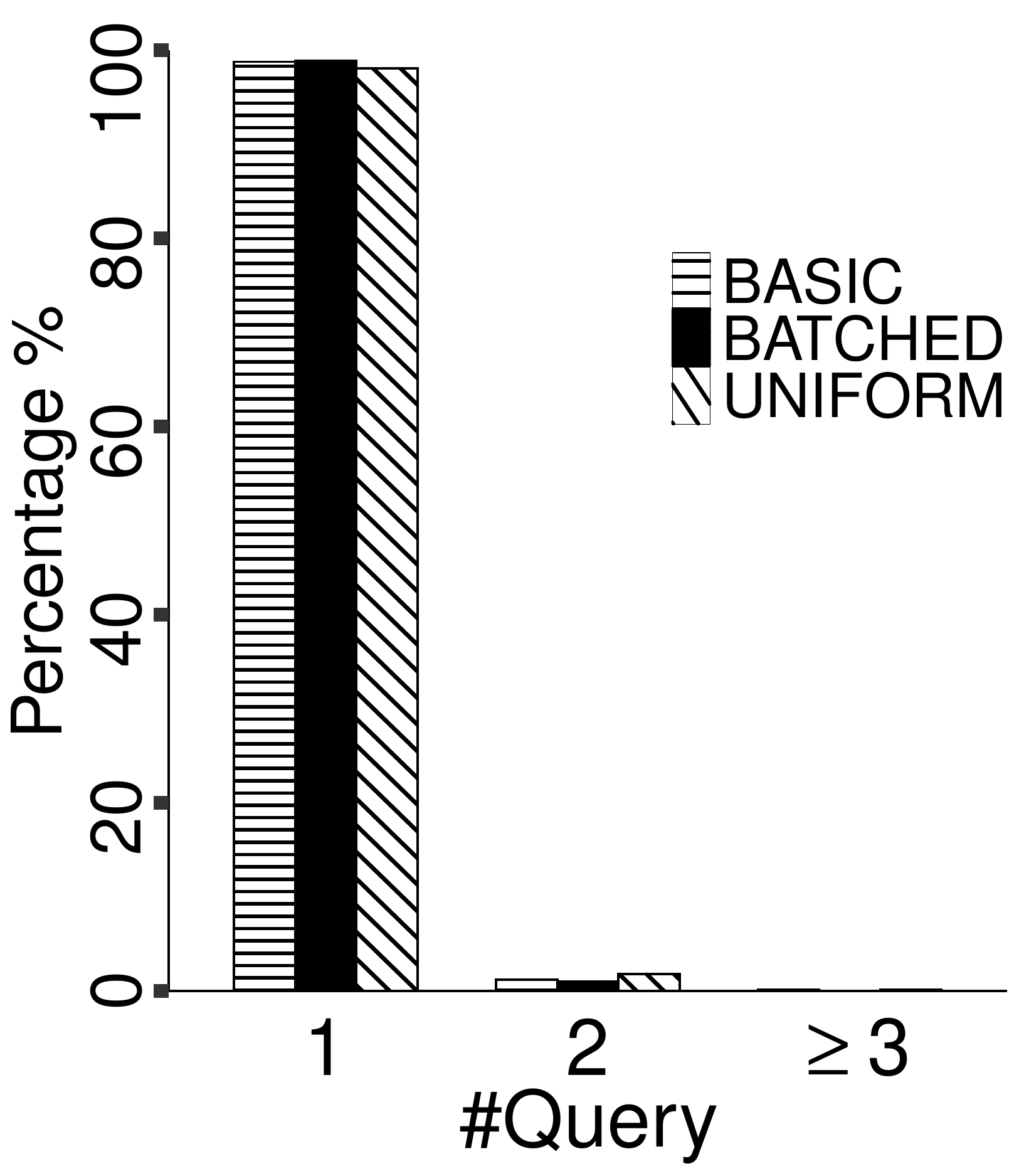}
\centerline{$p=0$}
\end{minipage}
\hfill
\begin{minipage}[d]{0.31\linewidth}
\includegraphics[width=\textwidth]{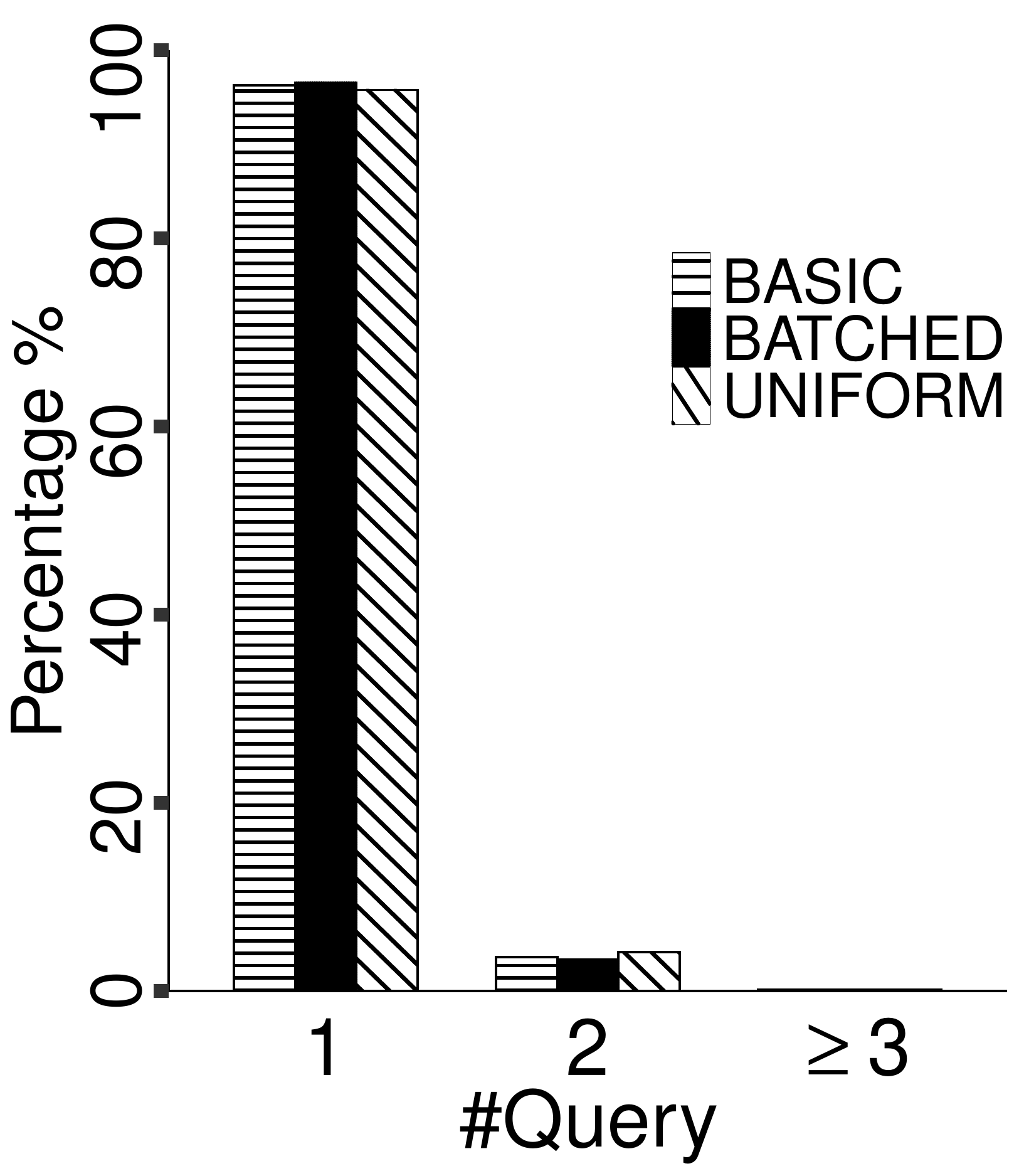}
\centerline{$p=0.1$}
\end{minipage}
\hfill
\begin{minipage}[d]{0.31\linewidth}
\includegraphics[width=\textwidth]{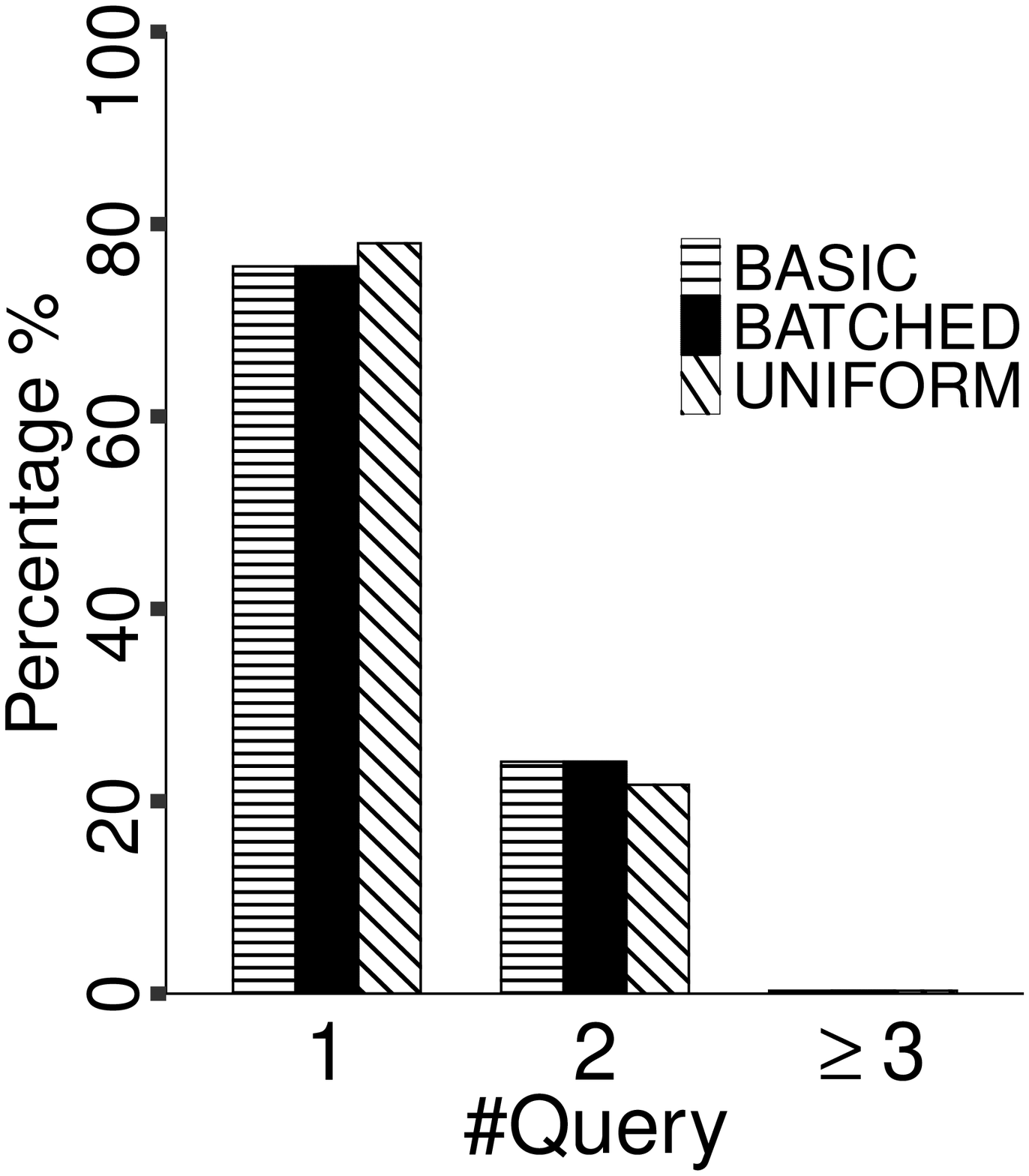}
\centerline{$p=0.3$}
\end{minipage}
\vspace{-0.8em}
\caption{Number of same-cluster queries on synthetic datasets.}
\label{fig:sim-querynum}
\end{figure}

\begin{figure}
\begin{minipage}[d]{0.45\linewidth}
\includegraphics[width=\textwidth]{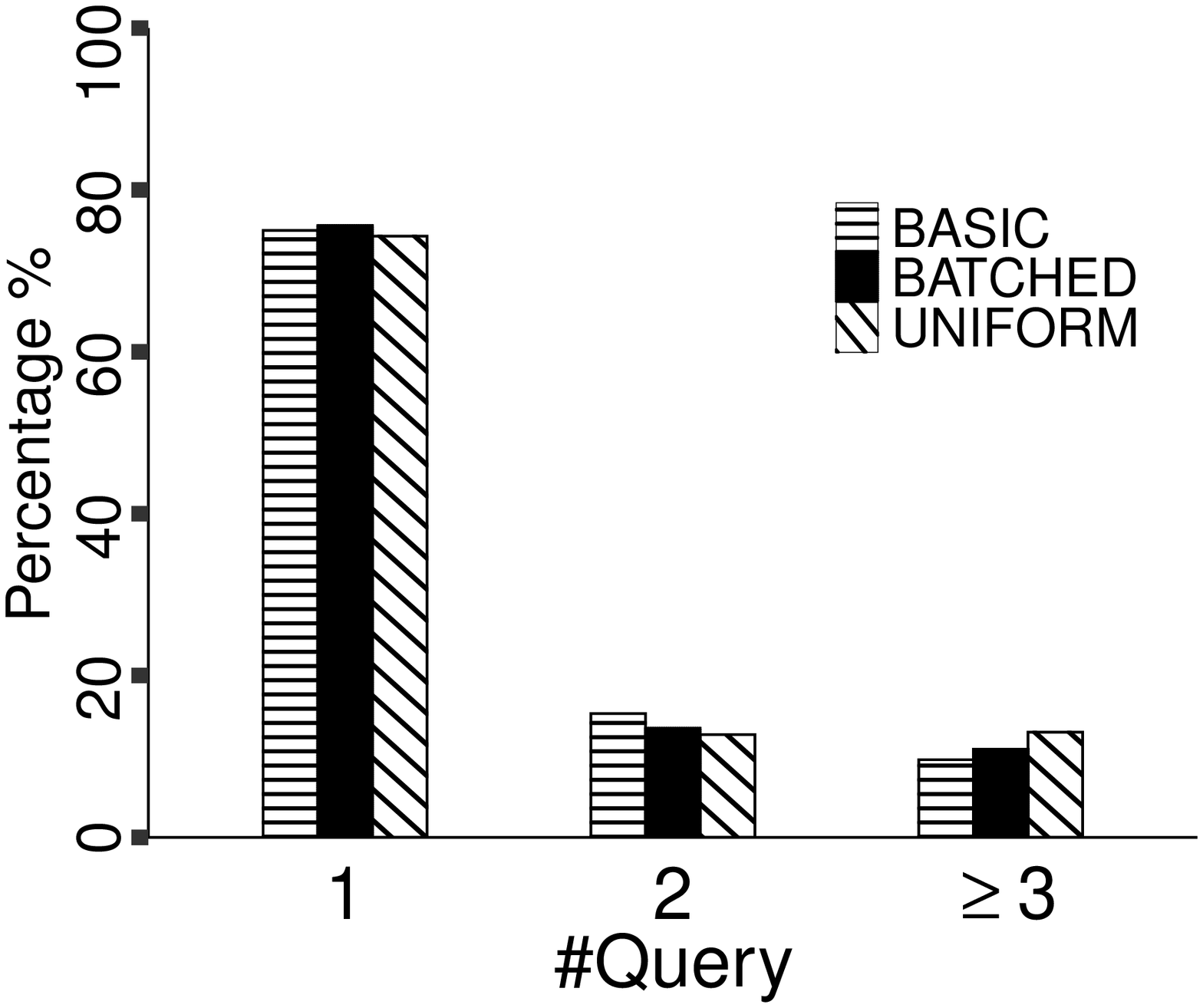}
\centerline{\shuttle}
\end{minipage}
\hfill
\begin{minipage}[d]{0.45\linewidth}
\includegraphics[width=\textwidth]{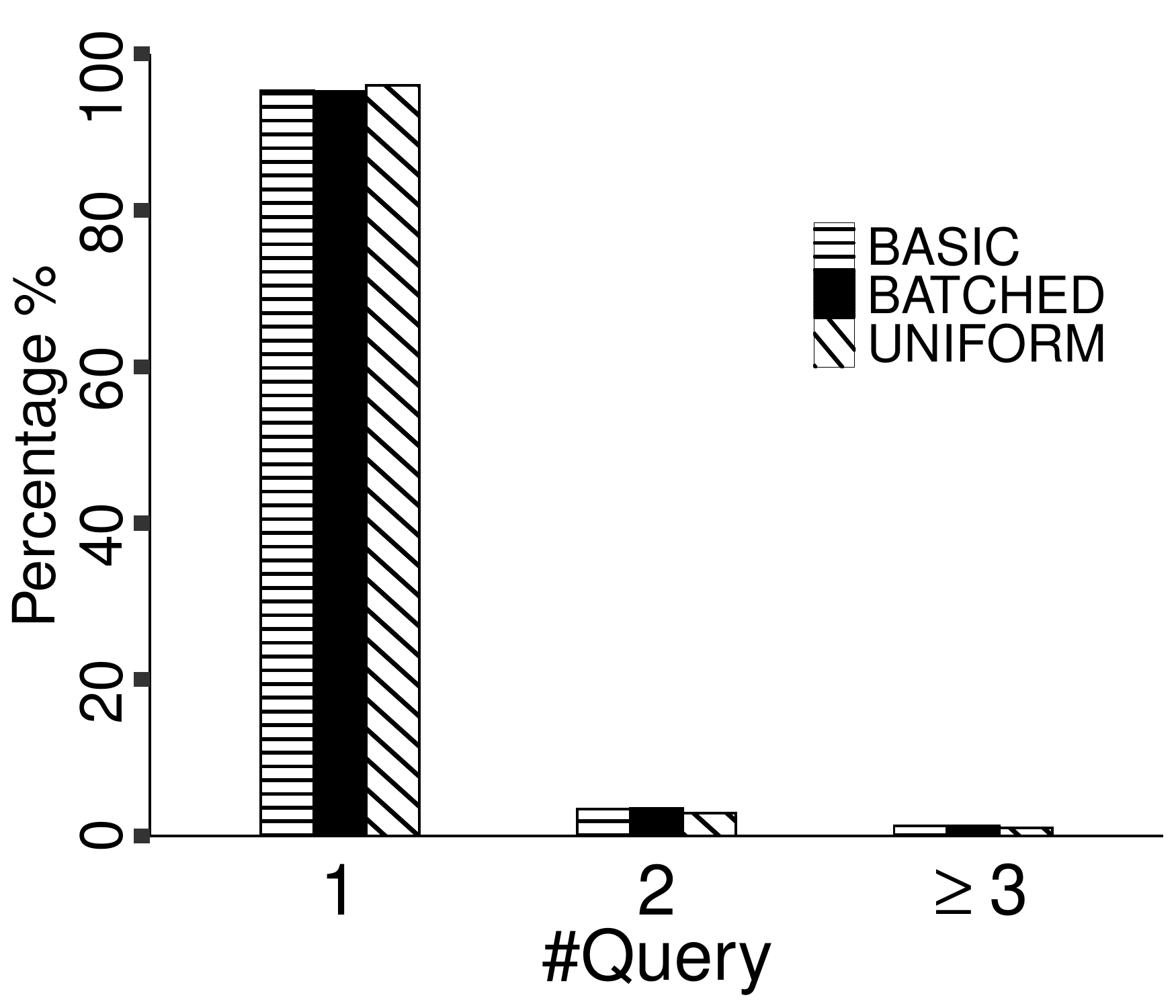}
\centerline{\kdd}
\end{minipage}
\vspace{-0.7em}
\caption{Number of same-cluster queries on real datasets.}
\label{fig:real-querynum}
\end{figure}

%% file: prelim_supp.tex
\section{Supplementary Preliminaries}

The following forms of Chernoff bounds will be used repeatedly in the analysis of our algorithms.

\begin{lemma}[Additive Chernoff]\label{lem:additive chernoff}
Let $X_1,\dots,X_N$ be i.i.d.\ Bernoulli random variables such that $\E X_i = p$. Let $\hat p = (\sum_{i=1}^N X_i)/N$. Then
$\Pr\left\{\hat p > p + t\right\} \leq e^{-Np}(ep/t)^{Nt}$.
\end{lemma}

\begin{lemma}[Multiplicative Chernoff]\label{lem:multiplicative chernoff}
Let $X_1,\dots,X_N$ be i.i.d.\ Bernoulli random variables such that $\E X_i = p$. Let $\hat p = (\sum_{i=1}^N X_i)/N$. Then for $\delta\in (0,1)$, it holds that $\Pr\left\{\hat p \leq (1-\delta) p\right\} \leq \exp(-\frac12 \delta^2 Np)$.
\end{lemma}

%% file: basic_analysis.tex




\section{Omitted Proofs in Section~\ref{sec:basic}}\label{sec:basic proofs}

We need two lemmata from~\cite{ABJK18}.

\begin{lemma}[{\cite{ABJK18}}]\label{lem:reference_bound}
It holds that 
\[
\frac{\Phi(\{x\}, C)}{\Phi(X_j, C)}\geq \frac{\eps}{64}\cdot\frac{1}{\abs{X_j}}.
\]
\end{lemma}

\begin{lemma}[{\cite{ABJK18}}]\label{lem:itcs_Y_j}
It holds that
\[
\frac{\Phi(Y_j, C)}{\Phi(X_j, C)}\geq \frac{\eps}{128} ,
\]
and
\[
\frac{\Phi(\{y\}, C)}{\Phi(\{x\}, C)}\leq \frac{\eps}{128}.
\]
\end{lemma}

\subsection{Proof of Lemma~\ref{lem:undiscovered_is_heavy}} \label{sec:proof_undiscovered_heavy}

\begin{proof}
We adapt the proof of Lemma 15 in \cite{ABJK18} to our irreducibility assumption. Observe that 
\begin{align*}
\sum_{i}\Phi(X_i, \tilde C) &= \sum_{i\in I}\Phi(X_i, \tilde C) + \sum_{i\not\in I}\Phi(X_i, \tilde C) \\
&\leq \sum_{i\in I}\Phi(X_i, \tilde C) + \frac{\eps/4}{1-\eps/4}\sum_{i\in I}\Phi(X_i, \tilde C)\\
&\leq \left(1 + \frac{\eps}{2}\right)\sum_{i\in I}\Phi(X_i, \tilde C)
\end{align*}
Splitting the sum on the left-hand side into $i\in I$ and $i\not\in I$,
\[
\sum_{i\in I} \Phi(X_i, \tilde C) + \sum_{i\not\in I}\Phi(X_i, \tilde C) \leq \left(1+\frac{\eps}{2}\right) \sum_{i\in I}\Phi(X_i, \tilde C),
\]
and so
\[
\sum_{i\not\in I}\Phi(X_i, \tilde C) \leq \frac{\eps}{2} \sum_{i\in I}\Phi(X_i, \tilde C)
\]
Comparing with the true centers, we have
\[
\sum_{i\not\in I}\Phi(X_i, C) \leq \frac{\eps}{2}(1+\eps)^2 \sum_{i\in I}\Phi(X_i, C)\leq \eps\sum_{i\in I}\Phi(X_i, C).
\]
By the $\eps$-irreducibility assumption, there exists an $\ell\not\in I$ such that
\[
\sum_{i\not\in I}\Phi(X_i, C) >  \Phi(X_\ell,C) > \eps \sum_{i\in I} \Phi(X_i,\mu_i),
\]
and so we have a contradiction. The claimed result in the lemma therefore holds.
\end{proof}

\subsection{Proof of Lemma~\ref{lem:correctness basic}}
\begin{proof}
By Lemma~\ref{lem:undiscovered_is_heavy}, the probability that no undiscovered clusters is seen among $T_1$ samples (Line~\ref{alg:basic T_1}) with probability
\[
\left(1-\frac{\eps}{4}\right)^{T_1} \leq e^{-2\ln(10(k+1))} = \frac{1}{100(k+1)^2}.
\]

The maintenance of the loop invariant \eqref{eqn:good_approx_centroid} in each round is a corollary of Lemma~\ref{lem:centroid}. With $T_3 = 20\eps^{-1}r\ln^2(10r)$ (see Line~\ref{alg:basic T_3}), the centroid estimate $\tilde \mu_j$ satisfies~\eqref{eqn:good_approx_centroid} with probability at least $1-1/(20r\ln^2(10r))$. 

Taking a union bound over all rounds, the total failure probability is at most
\[
\sum_{i=1}^\infty \frac{1}{100i^2} + \sum_{r=1}^\infty \frac{1}{20r\ln^2(10r)} \leq 0.05.\qedhere
\]
\end{proof}

\subsection{Proof of Lemma~\ref{lem:each p_i basic}}
\begin{proof}
Upon the termination of the while loop in Phase 2, let $s'$ denote the number of samples belonging to the clusters in $Q$. We claim that $s'\geq 6q\ln(10(k+q))$ with probability at least $1-1/(100(k+q)^2)$. Observe that $s \geq 96q\ln(10(k+q))/\eps$ at this point, then it follows from a Chernoff bound (Lemma~\ref{lem:multiplicative chernoff}) that
\[
\Pr\{s' < 6\ln(10(k+q))\} \leq \exp\left(-\frac{1}{2}\cdot \frac{1}{4}\cdot \frac{s\eps}{4}\right) \leq \frac{1}{100(k+q)^2}.
\]

We choose $j$ to be the newly discovered cluster of most samples and so $\hat p_j \geq 1/q$.
If $p_j < 1/(3q)$, then by a Chernoff bound (Lemma~\ref{lem:additive chernoff}),
\[
\Pr\left\{ \hat p_j \geq \frac{1}{q} \right\} \leq \left(\frac{\frac{1}{3}e}{1-\frac{1}{3}}\right)^{s' \frac{2}{3q}} \leq \frac{1}{100(k+q)^2}.\qedhere
\]
\end{proof}

\subsection{Proof of Lemma~\ref{lem:Y_j basic}}
\begin{proof}
It suffices to show that each cluster $j\in W$ sees a sample in $Y_j$. The probability that a sample lies in $Y_j$ is
\[
\frac{\Phi(Y_j,C)}{\Phi(X_j,C)}\cdot \frac{\Phi(X_j,C)}{\Phi(X,C)} \geq \frac{\eps}{128}\cdot \left(\frac{\eps}{4}\cdot\frac{1}{3q}\right) = \frac{\eps^2}{2^{9}\cdot 3q},
\]
where we used Lemmata~\ref{lem:undiscovered_is_heavy} and~\ref{lem:reference_bound}.

Therefore if we sample $T_2 = 2^{10}\cdot 3q\ln(10(k+q))/\eps^2$ points in total (Line~\ref{alg:basic T_2}), cluster $j$ does not see a point in $Y_j$ with probability at most
\[
\left(1 - \frac{\eps^2}{2^{9}\cdot 3q}\right)^{T_2} \leq e^{-2\ln(10(k+q))} \leq \frac{1}{100(k+q)^2}.\qedhere
\]
\end{proof}

\subsection{Proof of Lemma~\ref{lem:rejection sampling basic}}
\begin{proof}
Let $j\in W$. The rejection sampling has a valid probability threshold (at most $1$) by the second part of Lemma~\ref{lem:itcs_Y_j}. The rejection sampling includes a sampled point in $X_j$ with probability 
\begin{align*}
&\quad\ \sum_{x\in X_j} \frac{\Phi(\{x\},C)}{\Phi(X,C)} \cdot \left(\frac{\eps}{128}\cdot \frac{\Phi(\{x_j^\ast\},C)}{\Phi(\{x\},C)}\right) \\
&= \abs{X_j}\cdot \frac{\eps}{128}\frac{\Phi(\{x_j^\ast\},C)}{\Phi(X,C)}\\
&= \abs{X_j}\cdot \frac{\eps}{128}\cdot \frac{\Phi(X_j,C)}{\Phi(X,C)}\cdot \frac{\Phi(\{x_j^\ast\},C)}{\Phi(X_j,C)}\cdot \\
&\geq \abs{X_j} \cdot \frac{\eps}{128}\cdot \frac{\eps}{4}\cdot \frac{1}{3q}\cdot \frac{\eps}{64|X_j|} \geq \frac{\eps^3}{2^{15}\cdot 3q},
\end{align*}
where the first inequality uses Lemma~\ref{lem:reference_bound} to lower bound the last factor.

Hence, by sampling $s = 2^{23} \eps^{-4}q r\ln^2(10r)$ points in total, we expect to see at least $s\eps^3/(2^{15}\cdot 3q) > 2 T_3$ points returned by the rejection sampling procedure. Using the multiplicative Chernoff bound (Lemma~\ref{lem:multiplicative chernoff}), the rejection sampling returns at most $T_3 = 20\eps^{-1}r\ln^2(10r)$ points with probability at most
\[
\exp\left( -\frac{1}{2}\cdot \frac{1}{4} \cdot \frac{s\eps^3}{2^{15}\cdot 3q} \right) \leq \exp\left(-8r\right).\qedhere
\]
\end{proof}

\subsection{Proof of Theorem~\ref{thm:basic}}
\begin{proof}
In each round, by Lemmata~\ref{lem:each p_i basic}, \ref{lem:Y_j basic} and \ref{lem:rejection sampling basic}, Phases 2--4 fails with probability at most $3/(100(k+q)^2) + \exp(-8r)$. Combining with Lemma~\ref{lem:correctness basic} and summing over all rounds, we see that the overall failure probability is at most
\[
0.05 + \sum_{i=1}^\infty \left(\frac{3}{100i^2} + \exp(-8i)\right) \leq 0.1.
\]

Let $q_r$ denote the number of newly discovered clusters in the $r$-th round. There are $K$ rounds in total and the total number of samples is therefore
\[
O\left(\sum_{r=1}^K \frac{q_r r\log^2 L}{\eps^4} \right) = O\left(\frac{K^2 L \log^2 L}{\eps^4}\right),
\]
where we used the fact that $q_r\leq L$. For each sample we need to call the same-cluster query $O(L)$ times to obtain the its cluster index, and thus the query complexity is $O(\eps^{-4}K^2 L^2 \log^2 L)$.
\end{proof}

%% file: omitted_analysis_proofs.tex
\section{Omitted Proofs in Section~\ref{sec:improve}}\label{sec:improve proofs}

\subsection{Proof of Lemma~\ref{lem:correctness}}
\begin{proof}
As argued in the proof of Lemma~\ref{lem:correctness basic}, the probability that no undiscovered clusters is seen among $T_1$ samples (see Line~\ref{alg:main T_1} in Algorithm~\ref{alg:main}) with probability $1/(100(k+1)^2)$.

Regarding the maintenance of the loop invariant \eqref{eqn:good_approx_centroid}, with $T_3 = 30K/\eps$ (see Line~\ref{alg:main T_2} in Algorithm~\ref{alg:main}), each centroid estimate $\tilde \mu_i$ satisfies~\eqref{eqn:good_approx_centroid} with probability at least $1-1/(30K)$. Taking a union bound, we know that \eqref{eqn:good_approx_centroid} holds for all $i\in W$ in each single round with probability at least $1-|W|/(30K)$.

Let $w_r$ denote the number of newly recovered clusters in the $r$-th round, then $\sum_r w_r \leq K$. Taking a union bound over all rounds as in Lemma~\ref{lem:correctness basic}, we see that the failure probability is at most
\[
\sum_{i=1}^\infty \frac{1}{100i^2} + \sum_r \frac{w_r}{30K} = \frac{\pi^2}{600} + \frac{1}{30} \leq 0.05.\qedhere
\]
\end{proof}

\subsection{Proof of Lemma~\ref{lem:each p_i}}
\begin{proof}
Let $s'$ denote the number of samples belonging to the clusters in $Q$.
We claim that upon the termination of the while loop in Phase 2, it holds that $s'\geq 100|W|\log q\ln(10(k+q))$ with probability at least $1-1/(100(k+q)^2)$. Observe that $s \geq 1600|W|\log q\ln(10(k+q))/\eps$ at this point, then it follows from a Chernoff bound (Lemma~\ref{lem:multiplicative chernoff}) that
\begin{align*}
\Pr\{s' < 100\ln(10(k+q))\} &\leq \exp\left(-\frac{1}{2}\cdot \frac{1}{4}\cdot \frac{s\eps}{4}\right)
\leq \frac{1}{100(k+q)^2}.
\end{align*}

Observe that every cluster $i$ in a heavy band $B_\ell$ satisfies that
\[
\hat p_i\geq \frac{\sum_{j\in B_\ell} p_j}{\frac{3}{2}|B_\ell|}\geq \frac{1}{6|B_{\ell}|L}\geq \frac{1}{18|B_\ell|\log q}.
\]
If $p_i < 1/(70|B_\ell|\log q)$, then by a Chernoff bound (Lemma~\ref{lem:additive chernoff}),
\begin{align*}
\Pr\left\{ \hat p_i \geq \frac{1}{18|B_\ell|\log q} \right\} \leq \left(\frac{\frac{1}{70}e}{\frac{1}{18}-\frac{1}{70}}\right)^{s'(\frac{1}{18}-\frac{1}{70})\frac{1}{|W|\log q}} 
\leq \frac{1}{100(k+q)^3}.
\end{align*}
Taking a union bound over all $|W|\leq q$ clusters yields the claimed result.
\end{proof}

\subsection{Proof of Lemma~\ref{lem:Y_j}}
\begin{proof}
It suffices to show that each cluster $j\in W$ sees a sample in $Y_j$. The probability that a sample lies in $Y_j$ is
\[
  \frac{\Phi(Y_j,C)}{\Phi(X_j,C)}\cdot \frac{\Phi(X_j,C)}{\Phi(X,C)} \geq \frac{\eps}{128} \left(\frac{\eps}{4}\!\cdot\!\frac{1}{70|B_{\ell(j)}|\log q}\right) = \frac{\eps^2}{2^{10}\!\cdot\! 35\!\cdot\! |B_{\ell(j)}|\log q}.
\]
Therefore if we sample $T_2 \geq 2^{10}\cdot 105\cdot |W|\log q\ln(5(k+q))/\eps^2$ points in total (see Line~\ref{alg:main T_2} in Algorithm~\ref{alg:main}), each cluster $j\in W$ does not see a point in $Y_j$ with probability at most
\begin{align*}
\left(1 - \frac{\eps^2}{2^{10}\cdot 35\cdot |B_{\ell(j)}|\log q}\right)^{T_2} \leq e^{-3\ln(5(k+q))} \leq \frac{1}{25(k+q)^3}
\end{align*}
and taking a union bound over all $j\in W$ yields the claim.
\end{proof}

\subsection{Proof of Lemma~\ref{lem:rejection}}
\begin{proof}
Let $j\in W$. The rejection sampling is a valid probability threshold (at most $1$) by the second part of Lemma~\ref{lem:itcs_Y_j}. The rejection sampling includes a sampled point in $X_j$ with probability 
\begin{align*}
  &\sum_{x\in X_j} \frac{\Phi(\{x\},C)}{\Phi(X,C)} \cdot \left(\frac{\eps}{128}\cdot \frac{\Phi(\{x_j^\ast\},C)}{\Phi(\{x\},C)}\right) \\
  &= |X_j|\cdot \frac{\eps}{128}\frac{\Phi(\{x_j^\ast\},C)}{\Phi(X,C)}\\
&= |X_j|\cdot \frac{\eps}{128}\cdot \frac{\Phi(X_j,C)}{\Phi(X,C)}\cdot \frac{\Phi(\{x_j^\ast\},C)}{\Phi(X_j,C)}\cdot \\
&\geq |X_j|\cdot \frac{\eps}{128}\cdot \frac{\eps}{4}\cdot \frac{1}{70|B_{\ell(j)}|\log q}\cdot \frac{\eps}{64|X_j|} \\
&\geq \frac{\eps^3}{2^{22} |B_{\ell(j)}|\log q} \\
&\geq \frac{\eps^3}{2^{22} |W|\log q}.
\end{align*}
Let $s = 2^{28}\eps^{-4} |W| K\log q$. By sampling $s$ points, for each $j\in W$, it follows from the multiplicative Chernoff bound (Lemma~\ref{lem:multiplicative chernoff}) that the rejection sampling returns at least $T_3 = 30K/\eps$ points (see Line~\ref{alg:main T_3} in Algorithm~\ref{alg:main}) with probability at most
\[
\exp\left( -\frac{1}{2}\cdot \frac{1}{4} \cdot \frac{s\eps^3}{2^{22}|W|\log q} \right) \leq \exp\left(-8K\right).
\]
Taking a union bound over all $j\in W$ yields the claimed result.
\end{proof}

\subsection{Proof of Theorem~\ref{thm:main}}
\begin{proof}
Consider the inner repeat-until loop for a fixed value of $K$. Let $w_r$ denote the number of newly recovered clusters in the $r$-th round, then $\sum w_r \leq K$. 

The overall failure probability is similar to that in the proof of Theorem~\ref{thm:basic}, and can be upper bounded by 
\begin{align*}
0.05 + \sum_{i=1}^\infty \frac{3}{100i^2} + \sum_r w_r \exp(-8K)
&= 0.05 + \frac{\pi^2}{200} + K\exp(-8K) \\
&\leq 0.1
\end{align*}
and the total number of samples is 
\[
O\left(\frac{\sum_{r=1}^R w_r K \log^2 L}{\eps^4} \right) = O\left( \frac{K^2 \log^2 L}{\eps^4} \right).
\]

Since there are $\log K$ repetitions and the number of samples is dominated by that in the last repetition, the total number of samples is $O(\eps^{-4}K^2 \log K\log^2 L)$. For each sample we need to call the same-cluster query $O(L)$ times to obtain the its cluster index, and thus the query complexity is $O(\eps^{-4}K^2 L \log K\log^2 L)$.
\end{proof}

%% file: noisy.tex

\section{Noisy Oracles}\label{sec:noisy oracle}
\input{noisyalg}

In this section we consider the noisy same-cluster oracle which errs with a constant probability $p<1/2$. We assume that the same-cluster oracle returns the same answer on repetition for the same pair of points. The immediate consequence is that the \textsc{Classify} function needs to be amended as we cannot keep one representative point $z_i$ from each cluster $i$ but have to keep a set of points $Z_i$ from  cluster $i$ to tolerate oracle error. We present the substitute for \textsc{Classify}, called \textsc{CheckCluster}, in Algorithm~\ref{alg:checkcluster}. The following is the guarantee of the \textsc{CheckCluster} algorithm.

\begin{lemma}\label{lem:checkcluster}
Suppose that $|I|\leq K$ and $M \leq C_1(K/\eps)^c$ for some absolute constants $C_1$ and $c$. There exists a constant $C = C(p,c) > 0$ such that when $|Z_i|\geq C K/\eps$ for all $i\in I$, all $M$ calls to the \textsc{CheckCluster} function is correct with probability at least $0.99$.
\end{lemma}
\begin{proof}
By a Chernoff bound (\ref{lem:multiplicative chernoff}), the algorithm yields the correct behaviour w.r.t. each cluster $i\in I$ with probability at least $1 - \exp(-(2p-1)^2pCK/(8\eps)) = 1- \exp(-C'K/\eps)$. Taking a union bound over all $i\in I$ and $N$ calls yields the overall failure probability at most 
$M |I|\exp(-C_p'K) \leq C_1(K/\eps)^{c+1}\exp(-C' K/\eps) \leq 0.01$, provided $C'$ is large enough. Then we let $C = 8C'/(p(2p-1)^2)$.
\end{proof}

In order to obtain the representative samples from each set, we deploy a reduction to the stochastic block model. This model considers a graph $G=(V,E)$ of $n$ nodes, which is partitioned as $V=V_1\cup V_2\cup\cdots \cup V_k$. An edge is added between two nodes belonging to the same partition with probability at least $1-p$ and between two nodes in different partitions with probability at most $q$. It is known how to find the (large) clusters in this model.

\begin{lemma}[{\cite{MS17a}}] \label{lem:SBM clustering}
There exists a polynomial time algorithm that, given an instance of a stochastic block model on $n$ nodes, retrieves all clusters of size $\Omega(\sqrt{n}\log n)$ using $O(n^3/(2p-1)^4)$ queries with probability $1-1/n$.
\end{lemma}

Our main theorem about the noisy same-cluster oracle is the following. We do not attempt to optimize the number of same-cluster queries.
\begin{theorem}
Suppose that the same-cluster oracle errs with a constant probability $p<1/2$. With probability at least $0.6$, Algorithm~\ref{alg:noisy} finds all the clusters and obtains for every $i$ an approximate centroid $\tilde\mu_i$ that satisfies~\eqref{eqn:good_approx_centroid}, using $\tilde{O}_p(\eps^{-6}L^6K)$ same-cluster queries in total.
\end{theorem}
\begin{proof}
The analysis is similar to that of the improved algorithm (Algorithm~\ref{alg:main}) and we shall only highlight the changes. For now let us assume that all \textsc{CheckCluster} calls return correct results (we shall verify this at the end of the proof).

The analysis of Phases 1 and 4 remains the same as the analysis for the improved algorithm (Algorithm~\ref{alg:main}).

In Phase 2, all clusters $j\in W$ satisfy $p_i = \Omega(1/(|W|\log q))$ and hence each such cluster receives in expectation at least $s'' = \Omega(|S|/(|W|\log q)\cdot(1/\eps))$ sampled points. We require $s''\geq \sqrt{|S|}\log|S|$ (by Lemma~\ref{lem:SBM clustering}), which leads to $|S| = \Omega(\eps^{-2} q^2\log^2 q\log^2(q/\eps))$. By a Chernoff bound, we see that $O(\eps^{-2}q^2\poly(\log((k+q)/\eps)))$ samples suffices to guarantee the correctness. However, we are unable to know the number of newly discovered clusters $q$ when we are taking samples. Hence we guess $q$ over the powers of $2$ sequentially and stop when we obtain more than $q/2$ new clusters of size at least $\Omega(\sqrt{|S|}\log|S|)$. For each guess of $q$, the number of same-cluster queries is $O_p(|S|^3) = \tilde{O}_p(\eps^{-6}q^6)$. There can be at most $\log K$ guesses and thus the number of queries in this phase is $\tilde{O}_p(\eps^{-6}q^6)$.

In Phase 3, we would need each cluster $j\in W$ to see in expectation $\Theta(1/\eps)$ samples. This is automatically achieved by the number of samples in the analysis above for Phase 2, with an appropriate adjustment of constants. 

For each fixed value of $K$, the number of calls to $\textsc{CheckCluster}$ is $O(K^2/\eps^4)$ and Lemma~\ref{lem:checkcluster} applies, verifying our earlier claim that all calls to $\textsc{CheckCluster}$ are correct. 

For each fixed value of $K$, in each round, Phase 2 uses $\tilde{O}(\eps^{-6}L^6)$ oracle calls (Lemma~\ref{lem:SBM clustering}) and there are $K$ rounds; each call to Algorithm~\ref{alg:checkcluster}
$\textsc{CheckCluster}$ uses $O(K^2/\eps)$ oracle calls and there are $O(K^2/\eps^4)$ calls to $\textsc{CheckCluster}$. Hence the total number of oracle calls for each fixed value of $K$ is $\tilde{O}(L^6 K/\eps^6) + O(K^2/\eps\cdot K^2/\eps^4) = \tilde{O}(L^6 K/\eps^6)$.

Overall there are $\log K$ guesses for the value of $K$, the overall number of oracle calls is $\tilde{O}(L^6 K/\eps^6)$.
\end{proof}

\begin{algorithm}[!t]
\caption{$\textsc{RejSampNoisy}(I,\{\tilde\mu_i\}_{i\in I},T,W,\{x_j^\ast\}_{j\in W})$: Rejection sampling on heavy clusters}\label{alg:rej_sampling_noisy}
\begin{algorithmic}[1]
		\Repeat
			\State $x\gets $ a point returned by $D^2$-sampling (w.r.t.\ $\{\tilde\mu_i\}_{i\in I}$)
			\State $j\gets \textsc{CheckCluster}(x, W)$
			\NIf{$j \neq null$ and $j\in W$}
				\State Add $x$ to $S_j$ with probability $\frac{\eps}{128}\cdot\frac{\Phi(\{x_j^\ast\},\{\tilde\mu_i\}_{i\in I})}{\Phi(\{x\},\{\tilde\mu_i\}_{i\in I})}$
		\Until{$|S_j|\geq T$ for all $j\in W$}
		\State \Return $\{S_j\}_{j\in W}$
\end{algorithmic}
\end{algorithm}

%% file: noisyalg.tex
\begin{algorithm}[!p]
\caption{Algorithm in the case of noisy oracle. The oracle errs with a constant probability $p<1/2$. The constants $C_1,C_2,\dots$ in the algorithm below all depend on $p$.}\label{alg:noisy}
\begin{algorithmic}[1]
\State $I\gets \emptyset$
\State $k\gets 0$ 
\State $r\gets 0$
\State $K\gets 1/2$
\Repeat
\State $K\gets 2K$
\Repeat
	\State $r\gets r+1$
	\State $Q\gets \emptyset$
	\State $S\gets \emptyset$
	\State $T_1\gets 8\eps^{-1}\ln(10(k+1))$\tikzmark{ph1a} 
	\State $S\gets \text{$D^2$-sample $T_1$ points w.r.t.\ $\{\tilde \mu_i\}_{i\in I}$}$ \tikzmark{right1}
	\State $new\gets false$
	\For{each $x\in S$}
		\State $j\gets \textsc{CheckCluster}(x,I)$
		\If{$j = null$}
			\State $new\gets true$
			\State break loop
		\EndIf
	\EndFor \tikzmark{ph1b}
	\If {$new = true$}
		\State $q \gets 1/2$ \tikzmark{ph2a}
		\Repeat 
			\State $q\gets 2q$
			\State $T \gets C_2\eps^{-2}q^2\poly(\log((k+q)/\eps)$
			\State $S\gets \text{$D^2$-sample $T$ points w.r.t.\ $\{\tilde \mu_i\}_{i\in I}$}$ \tikzmark{right2}
			\State $(Q,(Z_i)_{i\in Q})\gets \textsc{FindClusters}(S)$
			\State $Q \gets \{i\in Q: |Z_i| = \Omega(\sqrt{T}\log T)\}$
		\Until $|Q|\geq q/2$	
		\For{each $i\in Q$}
			\State $\hat p_i \gets |Z_i|/T$
		\EndFor
		\State Split the clusters in to bands $\{B_\ell\}$ 
		\State $W\gets$ all clusters in heavy bands\tikzmark{ph2b}
		\State $q\gets |Q|$ \tikzmark{ph3a}
		\For{each $j\in W$}
			\State $x_j^\ast \gets \argmin_{x\in Y_j} \Phi(\{x\}, \{\tilde\mu_i\}_{i\in I})$ \tikzmark{right3}
			\State $S_j\gets\emptyset$ 
		\EndFor \tikzmark{ph3b}
		\State $w_r\gets |W|$
		\State $T_3 \gets 30K/\eps$  \tikzmark{ph4a}
		\State $\{S_j\}_{j\in W} \gets \textsc{RejSampNoisy}(I,\{\tilde\mu_i\}_{i\in I},\ \tikzmark{right4} \newline 
									\hspace*{10em}T_3,W,\{x_j^\ast\}_{j\in W})$ 
		\For{each $j\in W$}
			\State $\tilde\mu_j \gets (1/|S_j|)\sum_{x\in S_j} x$
		\EndFor \tikzmark{ph4b} 
		\State $k\gets k+|W|$
		\State $I\gets I\cup W$
		\State Retain $C_4 K/\eps$ points in each $Z_j$ for all $j\in W$
	\EndIf
\Until{$Q=\emptyset$ or $k\geq K$}
\Until{$Q=\emptyset$ and $k\leq K$}
\AddNoteTight{ph1a}{ph1b}{right1}{Phase\\ 1}
\AddNoteTight{ph2a}{ph2b}{right2}{Phase\\ 2}
\AddNoteTight{ph3a}{ph3b}{right3}{Phase\\ 3}
\AddNoteTight{ph4a}{ph4b}{right4}{Phase\\ 4}
\end{algorithmic}
\end{algorithm}

\begin{algorithm}[t]
\caption{$\textsc{CheckCluster}(x,I)$: Check if a point $x$ belongs to some cluster $i\in I$ using the set of representative points $\{Z_i\}_{i\in I}$. 
}\label{alg:checkcluster}
\begin{algorithmic}[1]
	\NFor{each $i\in I$}
		\NIf{$x$ belongs to the same cluster\\ as  the majority points in $Z_i$}
			\State \Return $i$
	\State \Return $null$
\end{algorithmic}
\end{algorithm}

%% file: query_complexity_tables.tex
\begin{table}[H]
\setlength\tabcolsep{1.5pt}\small
\centering
\begin{tabular}{|c|l|r|r|r|r|r|r|r|r|r|}
\hline
\#Queries    & \multicolumn{1}{c|}{Algo.} & \multicolumn{1}{c|}{1.0E4} & \multicolumn{1}{c|}{2.5E4} & \multicolumn{1}{c|}{5.0E4} & \multicolumn{1}{c|}{7.5E4} & \multicolumn{1}{c|}{1.0E5} & \multicolumn{1}{c|}{1.2E5} & \multicolumn{1}{c|}{1.5E5} & \multicolumn{1}{c|}{1.75E5} & \multicolumn{1}{c|}{2.0E5} \\ \hline
$p=0$ & \baseline\                  & 16.96                         & 24.88                         & 36.95                         & 43.51                         & 49.34                         & 54.53                         & 60.47                         & 64.62                         & 68.25                         \\
              & \improved\                   & 17.45                         & 26.00                            & 37.98                         & 45.56                         & 51.92                         & 56.56                         & 62.85                         & 66.59                         & 70.68                         \\
              & \uniform\                   & 12.97                         & 19.30                          & 25.75                         & 30.73                         & 34.77                         & 37.59                         & 41.62                         & 44.48                         & 48.01                         \\ \hline
$p=0.1$ & \baseline\                  & 17.78                         & 25.10                          & 33.90                          & 40.52                         & 46.93                         & 51.14                         & 56.66                         & 60.84                         & 65.05                         \\
              & \improved\                   & 18.31                         & 25.85                         & 35.27                         & 43.28                         & 49.60                          & 53.78                         & 59.14                         & 63.50                          & 68.55                         \\
              & \uniform\                   & 13.11                         & 19.40                          & 25.88                         & 30.63                         & 34.73                         & 37.68                         & 41.73                         & 44.41                         & 47.54                         \\ \hline
$p=0.3$ & \baseline\                  & 17.26                         & 26.38                         & 33.73                         & 40.22                         & 45.62                         & 49.03                         & 54.09                         & 57.94                         & 60.65                         \\
              & \improved\                   & 17.91                         & 27.59                         & 35.44                         & 42.56                         & 47.74                         & 51.53                         & 56.56                         & 60.04                         & 62.99                         \\
              & \uniform\                   & 13.17                         & 19.26                         & 25.95                         & 31.17                         & 34.77                         & 38.11                         & 41.72                         & 45.04                         & 47.87                         \\ \hline
\end{tabular}
\caption{Performance comparison on synthetic datasets, fixed-budget}
\label{tab:sim-query}
\end{table}

\begin{table}[H]
\setlength\tabcolsep{1.5pt}\small
\centering
\begin{tabular}{|c|l|r|r|r|r|r|r|}
\hline
\#Clusters   & \multicolumn{1}{c|}{Algo.} & \multicolumn{1}{c|}{20} & \multicolumn{1}{c|}{30} & \multicolumn{1}{c|}{40} & \multicolumn{1}{c|}{50} & \multicolumn{1}{c|}{60} & \multicolumn{1}{c|}{70} \\ \hline
$p=0$ & \baseline\                  & 789.57                  & 1128.73                 & 1459.13                 & 1978.24                 & 2416.48                 & 2997.13                 \\
              & \improved\                   & 722.21                  & 1050.19                 & 1363.88                 & 1846.77                 & 2316.00                    & 2845.01                 \\
              & \uniform\                   & 1275.11                 & 2260.90                  & 3359.38                 & 4338.60                  & 5237.93                 & 6261.74                 \\ \hline
$p=0.1$ & \baseline\                  & 618.83                  & 1233.11                 & 1787.45                 & 2262.58                 & 2778.48                 & 3235.86                 \\
              & \improved\                   & 615.86                  & 1160.61                 & 1625.75                 & 2053.64                 & 2596.92                 & 3059.33                 \\
              & \uniform\                   & 1284.39                 & 2297.05                 & 3336.56                 & 4366.22                 & 5265.16                 & 6187.66                 \\ \hline
$p=0.3$ & \baseline\                  & 678.37                  & 1109.64                 & 1803.80                  & 2438.27                 & 3198.09                 & 4447.11                 \\
              & \improved\                   & 630.94                  & 1056.14                 & 1664.32                 & 2269.10                  & 2954.42                 & 4089.53                 \\
              & \uniform\                   & 1264.46                 & 2201.10                  & 3317.04                 & 4307.54                 & 5239.21                 & 6074.59                 \\ \hline
\end{tabular}
\caption{Performance comparison on synthetic datasets, fixed-recovery }
\label{tab:sim-cluster}
\end{table}

\begin{table}[H]
\setlength\tabcolsep{1pt}\small
\begin{minipage}{0.45\linewidth}
\begin{tabular}{|r|rrr|}
\hline
\multicolumn{1}{|c|}{\#Queries} & \multicolumn{1}{c}{ \baseline\ } & \multicolumn{1}{c}{ \improved\ } & \multicolumn{1}{c|}{ \uniform\ } \\ \hline
5.00E+03                         & 8.33                     & \textbf{8.89}                    & 5.03                     \\
2.00E+04                         & 13.08                    & \textbf{13.65}                   & 6.54                     \\
4.00E+04                         & 16.54                    & \textbf{17.05}                   & 8.21                     \\
6.00E+04                         & 17.20                     & \textbf{17.64}                   & 9.23                     \\
8.00E+04                         & 17.50                     & \textbf{18.07}                   & 9.79                     \\
1.00E+05                         & 17.80                     & \textbf{18.25}                   & 9.95                     \\ \hline
\end{tabular}
\centerline{Fixed-budget}
\end{minipage}
\hfill
\begin{minipage}{0.5\linewidth}
\begin{tabular}{|c|rrr|}
\hline
\multicolumn{1}{|c|}{\#Clusters} & \multicolumn{1}{c}{ \baseline } & \multicolumn{1}{c}{ \improved } & \multicolumn{1}{c|}{ \uniform } \\ \hline
6                                 & 210.68                   & \textbf{192.10}                   & 986.56                   \\
8                                 & 420.47                   & \textbf{351.37}                  & 3624.32                  \\
10                                & 861.65                   & \textbf{747.29}                  & 6694.91                  \\
12                                & 1252.59                  & \textbf{996.43}                  & 79672.90                  \\
14                                & 1838.29                  & \textbf{1475.32}                 & 142161.65                \\
16                                & 2484.36                  & \textbf{1901.43}                 & 201012.82                \\ \hline
\end{tabular}
\centerline{Fixed-recovery}
\end{minipage}
\vspace{0em}
\caption{Performance comparison on \kdd\ dataset}
\label{tab:kdd}
\end{table}

\begin{table}[H]
\setlength\tabcolsep{1pt}\small
\begin{minipage}[d]{0.45\linewidth}
\begin{tabular}{|c|rrr|}
   \hline
  \#Queries & \baseline & \improved  & \uniform  \\ \hline
  1.00E+03 & 3.22 & 3.22 & 3.00    \\
  5.00E+03 & 4.89 & \textbf{4.92} & 3.51 \\
  1.00E+04 & 5.48 & \textbf{5.55} & 4.06 \\
  2.00E+04 & 6.14 & \textbf{6.31} & 4.83 \\
  3.00E+04 & 6.55 & \textbf{6.61} & 5.13 \\ \hline
 \end{tabular}
\centerline{Fixed-budget}
\end{minipage}
\hfill
\begin{minipage}[d]{0.5\linewidth}
\begin{tabular}{|c|rrr|}
  \hline
  \#Clusters & \baseline    & \improved     & \uniform     \\ \hline
  3      & 86.03   & 86.65   & \textbf{79.10}    \\
  4      & 443.60   & \textbf{435.29}  & 1233.67 \\
  5      & 725.46  & \textbf{711.47}  & 3154.66 \\
  6      & 2053.65 & \textbf{2036.89} & 6988.62 \\
  7      & 4050.22 & \textbf{4050.03} & 7799.98 \\ \hline
\end{tabular}
\centerline{Fixed-recovery}
\end{minipage}
\vspace{0em}
\caption{Performance comparison on \shuttle\ dataset}
\label{tab:shuttle}
\end{table}

%% file: main.bbl

\begin{thebibliography}{19}


\ifx \showCODEN    \undefined \def \showCODEN     #1{\unskip}     \fi
\ifx \showDOI      \undefined \def \showDOI       #1{#1}\fi
\ifx \showISBNx    \undefined \def \showISBNx     #1{\unskip}     \fi
\ifx \showISBNxiii \undefined \def \showISBNxiii  #1{\unskip}     \fi
\ifx \showISSN     \undefined \def \showISSN      #1{\unskip}     \fi
\ifx \showLCCN     \undefined \def \showLCCN      #1{\unskip}     \fi
\ifx \shownote     \undefined \def \shownote      #1{#1}          \fi
\ifx \showarticletitle \undefined \def \showarticletitle #1{#1}   \fi
\ifx \showURL      \undefined \def \showURL       {\relax}        \fi
\providecommand\bibfield[2]{#2}
\providecommand\bibinfo[2]{#2}
\providecommand\natexlab[1]{#1}
\providecommand\showeprint[2][]{arXiv:#2}

\bibitem[\protect\citeauthoryear{Ailon, Bhattacharya, and Jaiswal}{Ailon
  et~al\mbox{.}}{2018a}]%
        {ABJ18}
\bibfield{author}{\bibinfo{person}{Nir Ailon}, \bibinfo{person}{Anup
  Bhattacharya}, {and} \bibinfo{person}{Ragesh Jaiswal}.}
  \bibinfo{year}{2018}\natexlab{a}.
\newblock \showarticletitle{Approximate Correlation Clustering Using
  Same-Cluster Queries}. In \bibinfo{booktitle}{\emph{LATIN}}.
  \bibinfo{pages}{14--27}.
\newblock


\bibitem[\protect\citeauthoryear{Ailon, Bhattacharya, Jaiswal, and Kumar}{Ailon
  et~al\mbox{.}}{2018b}]%
        {ABJK18}
\bibfield{author}{\bibinfo{person}{Nir Ailon}, \bibinfo{person}{Anup
  Bhattacharya}, \bibinfo{person}{Ragesh Jaiswal}, {and} \bibinfo{person}{Amit
  Kumar}.} \bibinfo{year}{2018}\natexlab{b}.
\newblock \showarticletitle{Approximate Clustering with Same-Cluster Queries}.
  In \bibinfo{booktitle}{\emph{ITCS}}. \bibinfo{pages}{40:1--40:21}.
\newblock


\bibitem[\protect\citeauthoryear{Arthur and Vassilvitskii}{Arthur and
  Vassilvitskii}{2007}]%
        {AV07}
\bibfield{author}{\bibinfo{person}{David Arthur} {and} \bibinfo{person}{Sergei
  Vassilvitskii}.} \bibinfo{year}{2007}\natexlab{}.
\newblock \showarticletitle{k-means++: the advantages of careful seeding}. In
  \bibinfo{booktitle}{\emph{SODA}}. \bibinfo{pages}{1027--1035}.
\newblock


\bibitem[\protect\citeauthoryear{Ashtiani, Kushagra, and Ben{-}David}{Ashtiani
  et~al\mbox{.}}{2016}]%
        {AKB16}
\bibfield{author}{\bibinfo{person}{Hassan Ashtiani}, \bibinfo{person}{Shrinu
  Kushagra}, {and} \bibinfo{person}{Shai Ben{-}David}.}
  \bibinfo{year}{2016}\natexlab{}.
\newblock \showarticletitle{Clustering with Same-Cluster Queries}. In
  \bibinfo{booktitle}{\emph{NIPS}}. \bibinfo{pages}{3216--3224}.
\newblock


\bibitem[\protect\citeauthoryear{Bressan, Cesa{-}Bianchi, Lattanzi, and
  Paudice}{Bressan et~al\mbox{.}}{2020}]%
        {BCLP20}
\bibfield{author}{\bibinfo{person}{Marco Bressan},
  \bibinfo{person}{Nicol{\`{o}} Cesa{-}Bianchi}, \bibinfo{person}{Silvio
  Lattanzi}, {and} \bibinfo{person}{Andrea Paudice}.}
  \bibinfo{year}{2020}\natexlab{}.
\newblock \showarticletitle{Exact Recovery of Mangled Clusters with
  Same-Cluster Queries}. In \bibinfo{booktitle}{\emph{Advances in Neural
  Information Processing Systems 33: Annual Conference on Neural Information
  Processing Systems 2020, NeurIPS 2020, December 6-12, 2020, virtual}},
  \bibfield{editor}{\bibinfo{person}{Hugo Larochelle},
  \bibinfo{person}{Marc'Aurelio Ranzato}, \bibinfo{person}{Raia Hadsell},
  \bibinfo{person}{Maria{-}Florina Balcan}, {and} \bibinfo{person}{Hsuan{-}Tien
  Lin}} (Eds.).
\newblock


\bibitem[\protect\citeauthoryear{Chien, Pan, and Milenkovic}{Chien
  et~al\mbox{.}}{2018}]%
        {CPM18}
\bibfield{author}{\bibinfo{person}{I~(Eli) Chien}, \bibinfo{person}{Chao Pan},
  {and} \bibinfo{person}{Olgica Milenkovic}.} \bibinfo{year}{2018}\natexlab{}.
\newblock \showarticletitle{Query K-means Clustering and the Double Dixie Cup
  Problem}. In \bibinfo{booktitle}{\emph{NeurIPS}}.
  \bibinfo{pages}{6650--6659}.
\newblock


\bibitem[\protect\citeauthoryear{Firmani, Saha, and Srivastava}{Firmani
  et~al\mbox{.}}{2016}]%
        {FSS16}
\bibfield{author}{\bibinfo{person}{Donatella Firmani}, \bibinfo{person}{Barna
  Saha}, {and} \bibinfo{person}{Divesh Srivastava}.}
  \bibinfo{year}{2016}\natexlab{}.
\newblock \showarticletitle{Online Entity Resolution Using an Oracle}.
\newblock \bibinfo{journal}{\emph{PVLDB}} \bibinfo{volume}{9},
  \bibinfo{number}{5} (\bibinfo{year}{2016}), \bibinfo{pages}{384--395}.
\newblock


\bibitem[\protect\citeauthoryear{Gamlath, Huang, and Svensson}{Gamlath
  et~al\mbox{.}}{2018}]%
        {GHS18}
\bibfield{author}{\bibinfo{person}{Buddhima Gamlath}, \bibinfo{person}{Sangxia
  Huang}, {and} \bibinfo{person}{Ola Svensson}.}
  \bibinfo{year}{2018}\natexlab{}.
\newblock \showarticletitle{Semi-Supervised Algorithms for Approximately
  Optimal and Accurate Clustering}. In \bibinfo{booktitle}{\emph{ICALP}}.
  \bibinfo{pages}{57:1--57:14}.
\newblock


\bibitem[\protect\citeauthoryear{Huleihel, Mazumdar, M{\'{e}}dard, and
  Pal}{Huleihel et~al\mbox{.}}{2019}]%
        {HMMP19}
\bibfield{author}{\bibinfo{person}{Wasim Huleihel}, \bibinfo{person}{Arya
  Mazumdar}, \bibinfo{person}{Muriel M{\'{e}}dard}, {and}
  \bibinfo{person}{Soumyabrata Pal}.} \bibinfo{year}{2019}\natexlab{}.
\newblock \showarticletitle{Same-Cluster Querying for Overlapping Clusters}. In
  \bibinfo{booktitle}{\emph{NeurIPS}}. \bibinfo{pages}{10485--10495}.
\newblock


\bibitem[\protect\citeauthoryear{Inaba, Katoh, and Imai}{Inaba
  et~al\mbox{.}}{1994}]%
        {IKI94}
\bibfield{author}{\bibinfo{person}{Mary Inaba}, \bibinfo{person}{Naoki Katoh},
  {and} \bibinfo{person}{Hiroshi Imai}.} \bibinfo{year}{1994}\natexlab{}.
\newblock \showarticletitle{Applications of Weighted Voronoi Diagrams and
  Randomization to Variance-Based \emph{k}-Clustering (Extended Abstract)}. In
  \bibinfo{booktitle}{\emph{SOCG}}. \bibinfo{pages}{332--339}.
\newblock


\bibitem[\protect\citeauthoryear{Jaiswal, Kumar, and Sen}{Jaiswal
  et~al\mbox{.}}{2014}]%
        {JKS14}
\bibfield{author}{\bibinfo{person}{Ragesh Jaiswal}, \bibinfo{person}{Amit
  Kumar}, {and} \bibinfo{person}{Sandeep Sen}.}
  \bibinfo{year}{2014}\natexlab{}.
\newblock \showarticletitle{A Simple ${D}^2$-Sampling Based {PTAS} for
  $k$-Means and Other Clustering Problems}.
\newblock \bibinfo{journal}{\emph{Algorithmica}} \bibinfo{volume}{70},
  \bibinfo{number}{1} (\bibinfo{year}{2014}), \bibinfo{pages}{22--46}.
\newblock


\bibitem[\protect\citeauthoryear{Kumar, Sabharwal, and Sen}{Kumar
  et~al\mbox{.}}{2010}]%
        {KSS10}
\bibfield{author}{\bibinfo{person}{Amit Kumar}, \bibinfo{person}{Yogish
  Sabharwal}, {and} \bibinfo{person}{Sandeep Sen}.}
  \bibinfo{year}{2010}\natexlab{}.
\newblock \showarticletitle{Linear-time approximation schemes for clustering
  problems in any dimensions}.
\newblock \bibinfo{journal}{\emph{J. {ACM}}} \bibinfo{volume}{57},
  \bibinfo{number}{2} (\bibinfo{year}{2010}), \bibinfo{pages}{5:1--5:32}.
\newblock


\bibitem[\protect\citeauthoryear{Mazumdar and Saha}{Mazumdar and Saha}{2017a}]%
        {MS17a}
\bibfield{author}{\bibinfo{person}{Arya Mazumdar} {and} \bibinfo{person}{Barna
  Saha}.} \bibinfo{year}{2017}\natexlab{a}.
\newblock \showarticletitle{Clustering with Noisy Queries}. In
  \bibinfo{booktitle}{\emph{NIPS}}. \bibinfo{pages}{5788--5799}.
\newblock


\bibitem[\protect\citeauthoryear{Mazumdar and Saha}{Mazumdar and Saha}{2017b}]%
        {MS17b}
\bibfield{author}{\bibinfo{person}{Arya Mazumdar} {and} \bibinfo{person}{Barna
  Saha}.} \bibinfo{year}{2017}\natexlab{b}.
\newblock \showarticletitle{Query Complexity of Clustering with Side
  Information}. In \bibinfo{booktitle}{\emph{NIPS}}.
  \bibinfo{pages}{4682--4693}.
\newblock


\bibitem[\protect\citeauthoryear{Saha and Subramanian}{Saha and
  Subramanian}{2019}]%
        {SS19}
\bibfield{author}{\bibinfo{person}{Barna Saha} {and} \bibinfo{person}{Sanjay
  Subramanian}.} \bibinfo{year}{2019}\natexlab{}.
\newblock \showarticletitle{Correlation Clustering with Same-Cluster Queries
  Bounded by Optimal Cost}. In \bibinfo{booktitle}{\emph{ESA}}.
  \bibinfo{pages}{81:1--81:17}.
\newblock


\bibitem[\protect\citeauthoryear{Verroios and Garcia{-}Molina}{Verroios and
  Garcia{-}Molina}{2015}]%
        {VG15}
\bibfield{author}{\bibinfo{person}{Vasilis Verroios} {and}
  \bibinfo{person}{Hector Garcia{-}Molina}.} \bibinfo{year}{2015}\natexlab{}.
\newblock \showarticletitle{Entity Resolution with crowd errors}. In
  \bibinfo{booktitle}{\emph{ICDE}}. \bibinfo{pages}{219--230}.
\newblock


\bibitem[\protect\citeauthoryear{Vesdapunt, Bellare, and Dalvi}{Vesdapunt
  et~al\mbox{.}}{2014}]%
        {VBD14}
\bibfield{author}{\bibinfo{person}{Norases Vesdapunt}, \bibinfo{person}{Kedar
  Bellare}, {and} \bibinfo{person}{Nilesh~N. Dalvi}.}
  \bibinfo{year}{2014}\natexlab{}.
\newblock \showarticletitle{Crowdsourcing Algorithms for Entity Resolution}.
\newblock \bibinfo{journal}{\emph{PVLDB}} \bibinfo{volume}{7},
  \bibinfo{number}{12} (\bibinfo{year}{2014}), \bibinfo{pages}{1071--1082}.
\newblock


\bibitem[\protect\citeauthoryear{Wang, Kraska, Franklin, and Feng}{Wang
  et~al\mbox{.}}{2012}]%
        {WKF12}
\bibfield{author}{\bibinfo{person}{Jiannan Wang}, \bibinfo{person}{Tim Kraska},
  \bibinfo{person}{Michael~J. Franklin}, {and} \bibinfo{person}{Jianhua Feng}.}
  \bibinfo{year}{2012}\natexlab{}.
\newblock \showarticletitle{CrowdER: Crowdsourcing Entity Resolution}.
\newblock \bibinfo{journal}{\emph{PVLDB}} \bibinfo{volume}{5},
  \bibinfo{number}{11} (\bibinfo{year}{2012}), \bibinfo{pages}{1483--1494}.
\newblock


\bibitem[\protect\citeauthoryear{Wang, Li, Kraska, Franklin, and Feng}{Wang
  et~al\mbox{.}}{2013}]%
        {WLK13}
\bibfield{author}{\bibinfo{person}{Jiannan Wang}, \bibinfo{person}{Guoliang
  Li}, \bibinfo{person}{Tim Kraska}, \bibinfo{person}{Michael~J. Franklin},
  {and} \bibinfo{person}{Jianhua Feng}.} \bibinfo{year}{2013}\natexlab{}.
\newblock \showarticletitle{Leveraging transitive relations for crowdsourced
  joins}. In \bibinfo{booktitle}{\emph{SIGMOD}}. \bibinfo{pages}{229--240}.
\newblock


\end{thebibliography}
